\documentclass[letter,journal]{IEEEtran}
\usepackage{amsmath,amssymb,amsfonts,amsthm}
\usepackage{algorithmic}
\usepackage{array}
\usepackage{textcomp}
\usepackage{stfloats}
\usepackage{url}
\usepackage{verbatim}
\usepackage{graphicx}

\usepackage{diagbox}
\usepackage{siunitx} 
\ifCLASSOPTIONcompsoc
    \usepackage[caption=false, font=normalsize, labelfont=sf, textfont=sf]{subfig}
\else
\usepackage[caption=false, font=footnotesize]{subfig}
\fi

\usepackage{cite}
\usepackage{multirow}
\usepackage{multicol}

\usepackage[utf8]{inputenc}
\usepackage{enumitem}

\usepackage{xcolor}
\usepackage{tabularx}
\usepackage{colortbl}

\usepackage{setspace}
\allowdisplaybreaks

\usepackage{tikz}

\definecolor{lime}{rgb}{0.88,2,10}

\usepackage[linesnumbered,ruled,vlined]{algorithm2e}

\SetKwComment{Comment}{$\triangleright$\ }{}

\usepackage{nomencl}
\makenomenclature
\renewcommand{\nompreamble}{The next list describes several symbols that will be later used within the body of the Article}
\makeatletter
\newcounter{subparagraph}[paragraph]
\renewcommand\thesubparagraph{%
  \theparagraph.\@arabic\c@subparagraph}
\newcommand\subparagraph{%
  \@startsection{subparagraph}    
    {5}                              
    {\parindent}                     
    {2ex \@plus 1ex \@minus .6ex} 
    {0em}                           
    {\normalfont\normalsize\bfseries}}

\usepackage{blindtext}
\usepackage[T1]{fontenc}
\usepackage[utf8]{inputenc}
\usepackage[spaces,hyphens]{xurl}
\newcommand*{\Resize}[2]{\resizebox{#1}{!}{$#2$}}%

\DeclareFontFamily{OT1}{pzc}{}
\DeclareFontShape{OT1}{pzc}{m}{it}{<-> s * [1.10] pzcmi7t}{}
\DeclareMathAlphabet{\mathpzc}{OT1}{pzc}{m}{it}
\renewcommand\labelenumi{(\roman{enumi})}
\renewcommand\theenumi\labelenumi

\makeatother
\newcommand{\fref}[1]{Fig.~\ref{#1}}

\newcommand{\sref}[1]{Section~\ref{#1}}
\usepackage[misc]{ifsym}

\SetKwInput{kwInit}{Initialize}

\usepackage{enumitem}

\usepackage{amsthm}
\newtheorem{asu}{Assumption}
\newcounter{subassumption}[asu]

\makeatletter
\renewcommand{\p@subassumption}{\theasu}
\makeatother

\newtheorem{theorem}{Theorem}

\newtheorem{lemma}{Lemma}

\let\oldnl\nl
\newcommand{\nonl}{\renewcommand{\nl}{\let\nl\oldnl}}

\usepackage{fancyhdr}
\fancypagestyle{mystyle}{
    \chead{\large IEEE Journal on Selected Areas in Communications (JSAC), 2024}}

\begin{document}
\title{{Communication-Efficient Federated Learning for LEO Satellite Networks Integrated with HAPs Using Hybrid NOMA-OFDM}}
  \author{Mohamed Elmahallawy,~\IEEEmembership{Student Member,~IEEE,}
      Tie Luo$^*$\thanks{$^*$Corresponding author.},~\IEEEmembership{Senior Member, IEEE}, and Khaled Ramadan
\thanks{M. Elmahallawy and T. Luo are with the Department of Computer Science, Missouri University of Science and Technology, USA (e-mail: \{meqxk, tluo\}@mst.edu). K. Ramadan is with the Department of Electronics and  Communication Engineering, The Higher Institute of Engineering in Elshorouk City, Cairo, Egypt (e-mail: k.ramadan@sha.edu.eg).\\
\indent This work was supported by the National Science Foundation (NSF) under Grant No. 2008878.} }



\maketitle
\thispagestyle{mystyle}
\begin{abstract}
Space AI has become increasingly important and sometimes even necessary for government, businesses, and society. An active research topic under this mission is integrating federated learning (FL) with satellite communications (SatCom) so that numerous low Earth orbit (LEO) satellites can collaboratively train a machine learning model. However, the special communication environment of SatCom leads to a very slow FL training process up to days and weeks. This paper proposes NomaFedHAP, a novel FL-SatCom approach tailored to LEO satellites, that (1) utilizes high-altitude platforms (HAPs) as distributed parameter servers ($\mathcal {PS}$s) to enhance satellite visibility, and (2) introduces non-orthogonal multiple access (NOMA) into LEO to enable fast and bandwidth-efficient model transmissions. In addition, NomaFedHAP includes (3) a new communication topology that exploits HAPs to bridge satellites among different orbits to mitigate the Doppler shift, and (4) a new FL model aggregation scheme that optimally balances models between different orbits and shells. Moreover, we (5) derive a closed-form expression of the outage probability for satellites in near and far shells, as well as for the entire system. Our extensive simulations have validated the mathematical analysis and demonstrated the superior performance of NomaFedHAP in achieving fast and efficient FL model convergence with high accuracy as compared to the state-of-the-art. 
\end{abstract}
\section{Introduction}

\IEEEPARstart{S}atellite communication (SatCom) technology is considered a major player of the Internet of Remote Things (IoRT) \cite{OliveBranch}. Recent advancements in SatCom have stimulated giant companies such as SpaceX, Amazon, and OneWeb, as well as government agencies such as ESA and NASA, to launch a large number of small satellites into space on low Earth orbits (LEOs) \cite{pachler2021updated}. These satellites are equipped with high-resolution cameras and collect massive satellite imagery \cite{zhai2023fedleo }. Meanwhile, the booming developments of AI have motivated the leverage of machine learning (ML) to perform satellite data analytics. However, traditional ML approaches which require downloading the massive imagery from satellites to a ground station (GS) are not practical due to the limited SatCom bandwidth. In this regard, federated Learning (FL) \cite{mcmahan2017communication, xie2020asynchronous} offers a potential solution as it replaces {\em data} transmission with {\em model} transmission, where each satellite trains an ML model onboard using its own data and then only sends the trained model to the GS. This substantially reduces the demand for bandwidth and gains the additional benefit of preserving data privacy.

However, integrating FL with SatCom is non-trivial and involves significant challenges. This is because FL entails an iterative training process that typically involves hundreds of communication rounds between clients and the parameter server ($\mathcal {PS}$). While this is not a big issue in usual networks, it is a critical bottleneck in SatCom/LEO, where clients are LEO satellites and the $\mathcal {PS}$ is the GS, due to four factors. First, the \textit{highly sporadic, intermittent, and irregular connectivity pattern} between LEO satellites and the GS. This is caused by the distinction of travel trajectories and speeds between LEO satellites and the $\mathcal{PS}$, where satellites orbit the earth with an {\em inclination angle} between 0--90$^o$ while $\mathcal{PS}$ travels along the Earth rotation direction. Second, the \textit{long propagation and transmission delays} in SatCom, due to the long distance and low data rate. Third, FL models are often large (e.g., 528/549MB for VGG16/19, 232MB for ResNet152, 479MB for EfficientNetV2L) because of the high image resolution and accuracy dictated by satellite applications such as national and homeland security, disaster and weather monitoring. Lastly, the wireless channels between LEO satellites and GS are unreliable due to adverse weather conditions like rain, fog, wind turbulence, and interference from other radio signals especially near the Earth's surface. As a result, the short visible time window between a satellite and the $\mathcal{PS}$ is often insufficient to allow the complete transmission of a model, especially when multiple satellites' transmissions are involved. Ultimately, the above challenges significantly impede the FL training process and result in slow convergence that takes days or even weeks \cite{chen2022satellite, razmi}. 

In this paper, we propose a novel FL-SatCom framework called \textit{NomaFedHAP} that is tailored to LEO satellites to address the above challenges. First, it introduces non-orthogonal multiple access (NOMA) as a spectrum-efficient communication scheme into FL-SatCom to enable satellites at different orbit shells (hence different altitudes) 
to transmit models concurrently from/to the $\mathcal {PS}$. 
Second, following our previous work \cite{happaper}, NomaFedHAP utilizes multiple high-altitude platforms (HAPs) in lieu of GSs to act as distributed $\mathcal {PS}$s to improve satellite visibility. But unlike \cite{happaper}, NomaFedHAP introduces a new satellite-HAP communication topology in which each HAP also serves as a {\em relay} to forward models among HAPs. This has the benefit of handling multiple orbits of satellites without inter-orbit communication, thus avoiding significant Doppler shifts. Third, NomaFedHAP proposes an FL model aggregation scheme that optimally balances the number of models between different orbits and shells prior to model aggregation, thereby avoiding a biased global model.

In summary, this paper makes the following contributions:
\begin{itemize}
    \item To the best of our knowledge, NomaFedHAP is the first work that introduces NOMA to FL-SatCom, addressing the link budget limit and enabling concurrent transmissions of large FL models within short and sporadic visible windows. It also uses HAPs instead of GS as $\mathcal {PS}$ to achieve faster convergence.

    \item We propose to use orthogonal frequency division multiplexing (OFDM) for communication among satellites {\em in the same orbit}. This results in a hybrid NOMA-OFDM communication scheme. Furthermore, we derive a closed-form expression of the outage probability for satellites located in near and far orbit shells, as well as for the entire system.
   
    \item We also propose (i) a new {\it satellite-HAP communication topology} in substitution of the traditional star topology of FL, where we let HAPs act as inter-orbit relays to mitigate the Doppler shift, (ii) a new \textit{intra-orbit model propagation} algorithm that utilizes inter-satellite links (ISL) and sub-orbital model aggregation to enable ``straggler'' satellites to participate in FL without waiting for their visible windows, and (iii) a new \textit{FL model aggregation} scheme that optimally balances the number of models between different orbits to avoid a biased global model.

    \item We extensively evaluate NomaFedHAP using 3 common datasets and a real satellite dataset, and demonstrate its efficient bandwidth utilization and high data rate when transmitting FL models between satellites and HAPs. We also show that NomaFedHAP accelerates FL convergence by an order of magnitude 
    as compared to state-of-the-art FL-SatCom algorithms (4 vs. 72 hours), while achieving the highest model accuracy.
\end{itemize}
\section{Related work}\label{sec:releated}
While research in FL-SatCom is still in a nascent stage, a decent number of interesting studies have recently appeared \cite{chen2022satellite, razmi,happaper,lin2022federated, chen2023edge, wu2022dsfl, e2023opt,razmi2022ground, razmi2022scheduling,so2022fedspace,wang2022fl, mAsyFLEO, wu2023fedgsm, elone,elmahallawy2023secure} and have made appealing progress. 

\textbf{Synchronous FL.} In synchronous FL, the $\mathcal{PS}$ (e.g., GS) has to wait until all LEO satellites complete training and send their local models to $\mathcal{PS}$ when they sequentially come into its visible zone. Only after that, the $\mathcal {PS}$ will proceed to aggregate those received models into a \textit{global model} and then start the next communication round by sending the global model back to all the satellites for retraining. In such synchronous FL, slow satellites or ``stragglers'', who have limited visibility to the $\mathcal {PS}$, will become the bottleneck of the training process, where fast satellites have to idly wait.

Despite this, synchronous FL as a simple (and hence desirable) approach has been applied to LEO constellations in several studies. For instance, \cite{chen2022satellite} adopted the traditional synchronous FL approach  (i.e, FedAvg \cite{mcmahan2017communication}) without any tailoring for LEO constellations and demonstrated that FL is more advantageous than a direct download of raw data to a centralized server for training. The study, however, did not take into account the satellite-GS visibility challenge which slows down the FL training processes significantly. To deal with this challenge, \cite{razmi} proposed an intra-orbit routing technique called FedISL specifically designed for FL-SatCom. It performs well when the $\mathcal{PS}$ is a GS situated either at the North Pole (NP) or a medium Earth orbit (MEO) satellite flying above the Equator (at an altitude of 20,000 km). In these cases, each satellite visits the $\mathcal{PS}$ at regular intervals, resulting in more frequent communication and hence faster convergence. However, when the $\mathcal{PS}$ is located at a different position, the convergence time increases significantly, taking several days instead of just a few hours.
Moreover, NP is an ideal location which is often unavailable in practice, and MEO would incur considerable Doppler shifts. To address this limitation, FedHAP \cite{happaper} introduces HAPs as a proxy of GS to LEO constellations without restriction on locations. However, the FL model still requires more than a day to converge due to the non-ideal $\mathcal {PS}$ locations. Lin et al. \cite{lin2022federated} proposed an approach to dynamically aggregate satellite models based on connection density, excluding stragglers in cases of sparse connections, and involving the collaboration of multiple GSs at distributed locations. However, ensuring model consistency at multiple GSs is nontrivial and incurs more overhead, and excluding straggler satellites can introduce bias in the global model towards frequently visible satellites. In \cite{chen2023edge}, a clustering and edge selection approach is proposed, where a GS selects an LEO server, clusters neighboring LEO clients with good channel quality, and then selects satellites within each cluster to participate in the training process. However, this may result in a biased model toward satellites with good channel quality. The authors of \cite{wu2022dsfl} proposed a decentralized learning paradigm, eliminating the need for a $\mathcal {PS}$. They utilize intra- and inter-orbit ISLs with a traditional communication scheme like OFDM \cite{el2019performance}, attempting to address convergence speed; however, this approach still requires a higher data rate due to the Doppler shift among different orbits.
Moreover, all the above studies evaluate their models on non-SatCom-related datasets, and the convergence time still takes long hours and days. Most recently, FedLEO \cite{e2023opt} enhances the FL convergence through the use of intra-plane model propagation and scheduling of sink satellites, and uses a real satellite dataset to test their model. On the other hand, it requires each satellite to run a scheduler to determine a sink satellite, leading to delays during model exchanges with the GS.

{\bf Asynchronous FL.}  Unlike synchronous FL, asynchronous FL allows the $\mathcal {PS}$ to receive only a \textit{subset} of the satellites' (typically fast ones') model updates to proceed to the next communication round. This mitigates the idleness problem in synchronous FL, but faces another problem called \textit{model staleness}, where outdated models from straggler satellites will arrive in later communication rounds, potentially affecting the model convergence and performance adversely (e.g., FedAsync \cite{xie2020asynchronous}).

Razmi et al. \cite{razmi2022ground} proposed FedSat, an asynchronous FL approach tailored for SatCom. FedSat adapts FedAvg \cite{mcmahan2017communication} to an asynchronous setting, where it averages the received satellite models based on their visibility order, assuming regular satellite visits to the GS. In response to this ideal consideration, they proposed another work \cite{razmi2022scheduling}, FedSatSchedule, that considers the location of the GS can be anywhere. FedSatSchedule is developed to reduce model staleness by determining whether each satellite has sufficient visible time for downloading the global model, training a local model, and uploading it. Otherwise, the satellite will schedule uploading its local model for the next communication round, allowing it to train its local model during its ``long'' invisible period. However, FedSatSchedule has only limited improvement in efficiency, still requiring several days for an FL model to converge. The authors of \cite{so2022fedspace} proposed another asynchronous FL approach, FedSpace, which stores satellite models into a buffer with a predicted size and reduces the weighting of stale models.  On the other hand, this method requires each satellite to upload a small amount of its raw data to the GS in order to be used for scheduling the model aggregation, which is contrary to the FL principle of maintaining privacy by not sharing raw data. Wang et al. \cite{wang2022fl} proposed a graph-based routing and resource reservation algorithm aimed at optimizing the delay in FL model parameter transfer. The algorithm enhances a storage time-aggregated graph, enabling a joint representation of the satellite networks' transmission, storage, and computing resources. AsyncFLEO \cite{mAsyFLEO} is a more recent asynchronous FL approach that offers a solution to the staleness challenge, where it first groups satellites from different orbits based on the similarity of their models, and then selects only fresh models from each group to include in the model aggregation. Outdated satellite models are only selected when it is the only model in a group and will be down-weighted during aggregation. Finally, Wu et al. \cite{wu2023fedgsm} proposed FedGSM, which implements a compensation mechanism to mitigate gradient staleness. FedGSM utilizes the deterministic and time-varying topology of the orbits to counteract the negative impact of staleness. However, their approach still requires a long time for convergence.

While considerable efforts have been made to accelerate the convergence of the FL-SatCom model and address the challenge of satellite-GS sporadic connectivity, no work has been proposed to formally address the issue of uploading large satellite ML models to the $\mathcal {PS}$ (e.g., GS, MEO, HAPs) within the satellite link budget limit and the sporadic short satellite visibility periods. In addition, most of the existing works rely on direct communication between the server (often a GS) and LEO satellites, not leveraging the potential benefits of utilizing HAPs in space to improve FL training. Moreover, balancing among orbits to mitigate bias was overlooked too. This work aims to fill these gaps.

\section{System Model and Problem Formulation}\label{sec:model}

Consider a general LEO satellite constellation consisting of $N$ orbital shells, where each shell consists of $L$ orbits, all at the same altitude $d_n$\footnote{The orbit shell consists of multiple orbits, each carrying a number of satellites at the same altitude.}. In a shell $n$, each orbit $l \in \mathcal{L}=\{l_1,l_2,...,l_n\}$ carries $K_{l}$ equally distributed homogeneous satellites with an inclination angle of $ \Theta_{l}$. Each satellite in the orbit $l$ has a unique ID and travels at a speed of $v_{k}=\sqrt{\frac{G_e M}{(r_{E}+d_n)}}$ with an orbital period of $T_{k}= \frac{2 \pi (r_{E}+d_n)}{v_{k}}$, where $r_{E}=6371~km$ is the Earth's radius, $G_e$ is the Earth gravitational constant, and $M$ is the mass of the Earth ($G_e M= 3.98\times 10^{14} m^{3}/s^{2}$). In addition, consider $H$ HAPs, where each HAP $\mathpzc{h}$ serves as a $\mathcal {PS}$ and communicates with a diverse set of satellites deployed in different orbits/shells and performs (partial or full) model aggregation. Furthermore, a HAP $\mathpzc{h}$ can perform NOMA among satellites located in different shells based on their distance.

For the communication to be established between any satellite $k\in \mathcal K$  and the $\mathcal {PS}$ (i.e., a set of HAPs), the line of sight (LoS) link between them must not be blocked by the Earth. In mathematical terms, this can be expressed as:\vspace{-0.3cm}
\begin{equation}
    \vartheta_{k,\mathcal {PS}}(t) \triangleq \angle (r_{\mathcal {PS}}(t), (r_{k}(t) - r_{\mathcal {PS}}(t))) \leq \frac{\pi}{2}-\vartheta_{min}
\end{equation}
where $r_{k}(t)$ and $r_{\mathcal {PS}}(t)$ are the trajectories of satellite ${k}$ and the $\mathcal {PS}$, respectively, and $\vartheta_{min}$ is the minimum elevation angle, a constant depends on each HAP's location. 
To account for communication between any satellite and the $\mathcal{PS}$, we consider two particular satellites that are the closest and the farthest to the $\mathcal{PS}$, referred to as the nearest satellite (NS) and the farthest satellite (FS), respectively. An illustration of our system model is given in Fig.~\ref{Earth_observation}.\vspace{-0.5cm}
\begin{figure}[!t]
     \centering
    {{\includegraphics[width=\linewidth]{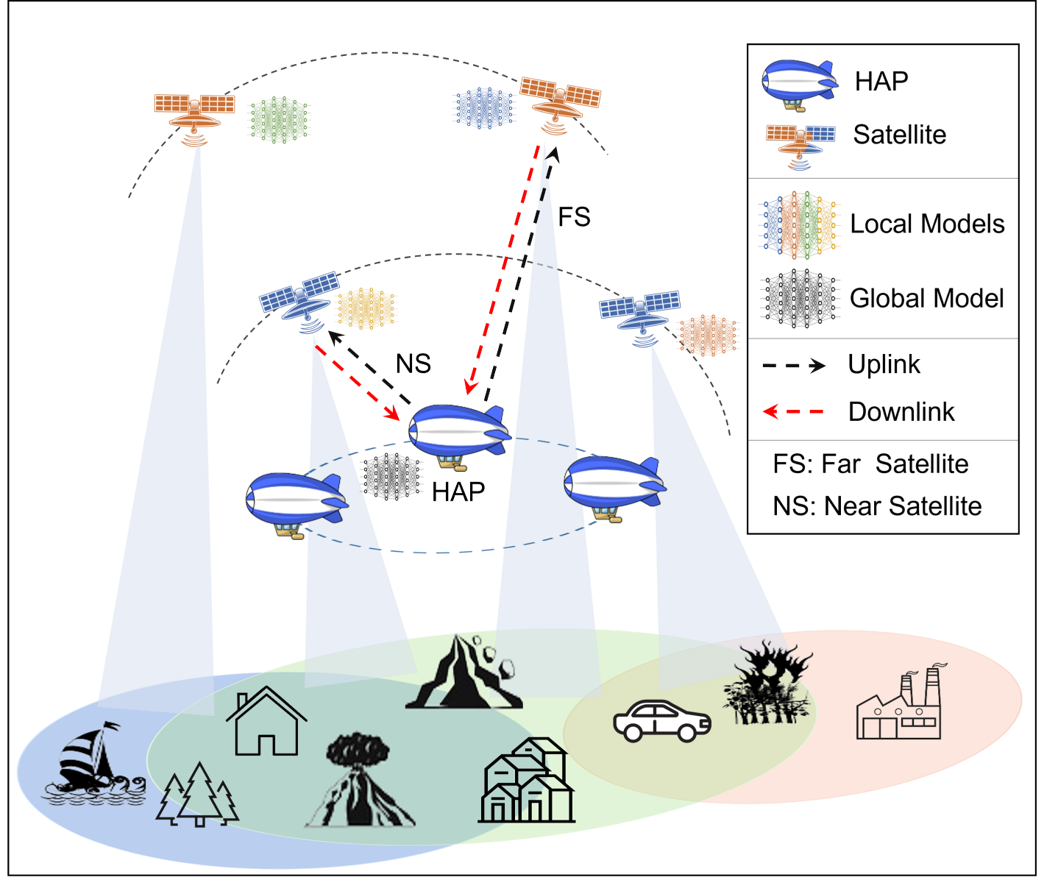}}}
    \caption{An FL system in the context of satellite constellations, using multiple HAPs as parameter servers. For the sake of clarity, only two orbit shells are shown. }
    \label{Earth_observation}
\end{figure}
\subsection{FL-LEO Computation Analysis}
Consider an LEO satellite constellation $\mathcal K$, where each satellite $k$ collects a set of Earth observational images. These images, collected by different satellites, are typically non-independent and identically distributed (non-IID) due to varying orbital speeds, altitudes, and coverage areas. Specifically, a satellite $k$ captures its own dataset $\mathcal {D}_{k}=\{(X_1, {\bf y}_1), (X_{2}, {\bf y}_{2}), \dots, ((X_{m_{k}}, {\bf y}_{m_{k}})\}$, where $X_{i}$ and ${\bf y}_{i}$ are the feature vector and its corresponding label of the \textit{i-th} sample (labels are not required but written here for notation purposes only). 
The overarching objective of the FL-LEO system is to have the LEO satellites and the $\mathcal {PS}$ work collaboratively to train a global ML model with the objective of minimizing the overall loss function $F(\boldsymbol{w})$:
\begin{equation} \label{eqn1}
          \arg \min_{\boldsymbol{w} \in \mathbb{R}^{d}} F(\boldsymbol{w})= \sum_{k\in \mathcal K }{\frac{|\mathcal {D}_{k}|}{|\mathcal {D}|}{F_{k}{(\boldsymbol{w})}},}
\end{equation}
where $\boldsymbol{w}$ is a vector representing the model weights, $|\mathcal {D}|= \sum_{i\in \mathcal K} |\mathcal {D}_{k}|$ is the total number of data samples collected by LEO satellites, and $F_{k}(\boldsymbol{w})$ is satellite $k$'s loss function:
\begin{equation}\label{eqn:loss}
     F_{k}{(\boldsymbol{w})} = \frac{1}{|\mathcal {D}_{k}|}\sum_{i=1}^{|\mathcal {D}_{k}|} f_{k}(\boldsymbol{w}; X_{i}, {\bf y}_{i}),
\end{equation}
where  $f_{k}(\boldsymbol{w}; X_{i}, {\bf y}_{i})$ is the training loss over a data point $X_{i}$.

The training process in a synchronous FL-LEO system such as FedHAP \cite{happaper} requires multiple communication rounds $\beta=0, 1, 2, \dots$, where $|\mathcal{B}|$ is the total number of communication rounds needed to achieve FL model convergence. During each round, the $\mathcal {PS}$ awaits each satellite $k$ to enter its visible zone (transiently) in order to send its aggregated global model $\boldsymbol{w}^\beta$ or to (subsequently) receive this satellite's local model $\boldsymbol{w}_{k}^\beta$.
What happens is that the satellite $k$ carries out a local optimization method such as stochastic gradient descent (SGD) on the received global model $\boldsymbol{w}^\beta$ using its local data, iterating $J$ local epochs to update the model:
\begin{equation}\label{eq:local_update}
    \boldsymbol{w}_{k}^{\beta,j+1} = \boldsymbol{w}_{k}^{\beta,j}- \zeta_{\beta} \nabla F_{k}(\boldsymbol{w}_{k}^{\beta,j}; X_{k}^{j}, {\bf y}_{k}^{j}), ~ j=1,2,...,J
\end{equation}
where $\zeta_{\beta}$ is the learning rate at the round $\beta$. After this training process, the satellite $k$ obtains an updated local model. It then transmits this model back to the $\mathcal {PS}$ when entering the $\mathcal {PS}$'s visible zone again. At the $\mathcal {PS}$, after it receives all the satellites' models, it aggregates them into a global model as
\begin{equation}
    \boldsymbol{w}^{\beta+1} = \sum_{k\in \mathcal K} \frac{|\mathcal {D}_{k}|}{|\mathcal {D}|} \boldsymbol{w}_{k}^{\beta,J}.
\end{equation}
The above procedure repeats, where $\beta$ continuously increases, until the FL model converges, i.e., achieves a target accuracy or loss, or reaches the maximum communication rounds. Fig.~\ref{FL_LEO_Graph} gives an illustration of the FL-LEO system.
\begin{figure}[!t]
\centering
\includegraphics [width=\linewidth]{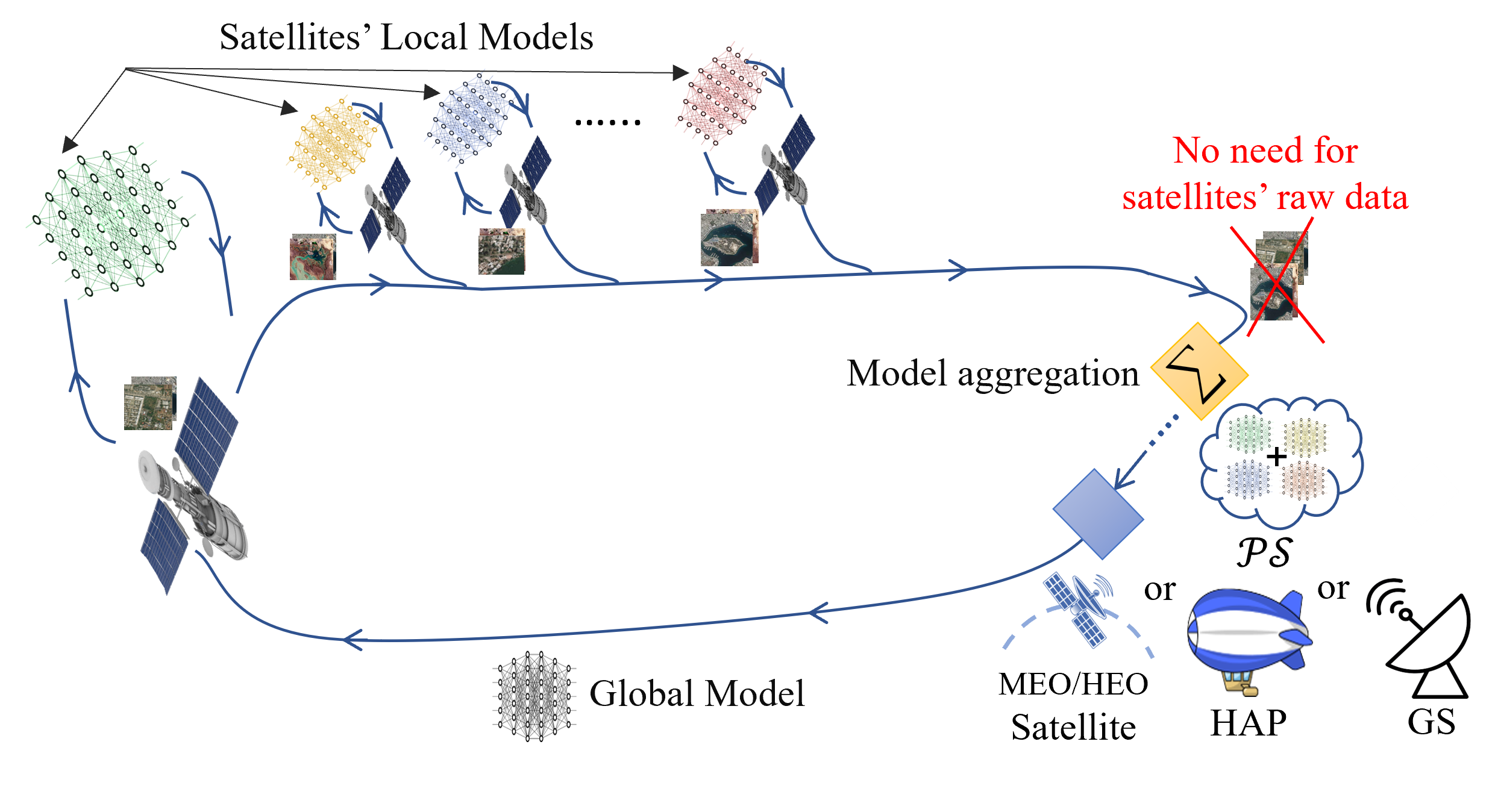}
\caption{Illustration of an FL-LEO system.}
\label{FL_LEO_Graph}\vspace{-0.2cm}
\end{figure}
One of the major challenges with this learning process is that all communications (uplink/downlink) can only occur when a satellite is transiently visible to the $\mathcal {PS}$, which significantly limits the communication opportunities and prolongs the entire process to several days or even weeks. This makes the learning speed unable to keep up with the rate at which data is collected by LEO satellites, and as a result, the global FL model is always outdated. Moreover, another challenge to FL-LEO convergence is that during the visible window of the $\mathcal {PS}$ for an LEO satellite, the limited bandwidth can be insufficient for a large satellite model to be fully uploaded to the $\mathcal {PS}$ and hence the satellite has to wait for the next longer visible window to start over, resulting in large delays.\vspace{-0.5cm}

\subsection{FL-LEO Communication Analysis}\label{Com_link}

In this paper, we consider a set of HAPs as the $\mathcal {PS}$ that coordinates the model aggregation process in the FL-LEO system. Because of this, we incorporate two types of communication links, satellite-HAP link (SHL) and inter-HAP link (IHL), in addition to the satellite-GS link and inter-satellite link (ISL) which exist in prior studies. We apply the Shadowed Rician channel fading model to analyze these links since there will be direct LoS and multi-path links during the visible periods, and the transmission between satellites can be affected by various factors including atmospheric conditions, rains, and obstacles or debris in space. Note that HAPs and LEO satellites can use free-space optical (FSO) links instead of radio frequency (RF) links to communicate at much higher data rates. However, we do not take this advantage in our simulation, so we can have a consistent setup with current state-of-the-art research and a fair comparison.

Our analysis of SHL takes into account four factors: i) free-space path loss, ii) antenna gain of the transmitter and receiver, iii) antenna pointing error and shadowing, and iv) fading of the channel. Consequently, the total link budget of SHL between a satellite $k$ and a HAP $\mathpzc{h}$, without small-scale fading, can be expressed as \cite{gamal2022performance}
\begin{equation} \label{eqn_6}
    SHL{(k,\mathpzc{h})} = \frac{G_{\mathpzc{h}}~G_{k}(\theta)}{{L}_{k,\mathpzc{h}}~L_p}, 
\end{equation}
where $G_{\mathpzc{h}}$ is the antenna gain of a HAP $\mathpzc{h}$, $G_{k}(\theta)$ is the beam gain of a satellite $k$,which given by \cite{yan2019ergodic}
\begin{equation}
    G_{k}(\theta)=G_{k} \Bigg(\frac{J_1(k_s)}{2k_s}+36\frac{J_3(k_s)}{(k_s)^3}\Bigg)^{2}
\end{equation}
where $G_k$ is the antenna gain of $k$, $J(\cdot)$ is the Bessel function, and $k_s$ is a constant denoting the distance between the beam of a satellite $k$ to a HAP $\mathpzc{h}$ and its entire coverage radius. ${L}_{k,\mathpzc{h}}$ is the free-space pass loss between a satellite $k$ and a HAP $\mathpzc{h}$. As long as a LoS link between them is established (i.e., not blocked by the Earth), ${L}_{k,\mathpzc{h}}$ can be given by \cite{happaper} 
\begin{equation}
    {L}_{k,\mathpzc{h}} = \bigg(\frac{4\pi \|k,\mathpzc{h}\|_{2}  f_c}{c}\bigg)^{2}
\end{equation}
where $\|k,\mathpzc{h}\|_{2}$ is the Euclidean distance between satellite  $k$ and HAP $\mathpzc{h}$, $f_c$ denotes the carrier frequency, and $c$ is the speed of light. In \eqref{eqn_6}, $L_p$ is the antenna pointing error loss, which can be expressed as \cite{gamal2022performance}
\begin{equation}
    L_p=2.7211\times 10^{-20}~{f_c}^2~{\theta_e}^2~{D_c}^2
\end{equation}
where ${\theta_e}$ denotes the angle of the pointing error, and ${D_c}$ is the diameter of the aperture antenna.

When using traditional communication schemes, such as orthogonal multiple access (OMA) systems, the total time $t_c$ required to exchange the local model generated by a satellite $w_k$ with the global model $\boldsymbol{w}$ generated by a HAP $\mathpzc{h}$ can be calculated as:
\begin{align}
    t_{c} &= t_{t}+t_{p}+t_{k}+ t_{\mathpzc{h}}, \label{eqnx}  \\t_{t} &= \frac{q{|\mathcal{D}|}}{R} , \quad t_{p} = \frac{\|k,\mathpzc{h}\|_{2}}{c}, \label{eqny}
\end{align}
where $t_t$ and $t_p$ are the transmission and propagation times, respectively, $t_{k}$ and $t_{\mathpzc{h}}$ are the processing delay at a satellite $k$ and a HAP $\mathpzc{h}$, respectively (we omit them in our simulation since they are much smaller than $t_t$ and $t_p$), ${|\mathcal{D}|}$ is the data samples of the dataset $\mathcal{D}$, $q$ is the number of bits in each sample, and $R$ is the maximum achievable data rate.

In OMA systems, since $R$ is limited by the bandwidth allocated to each satellite, the time required to transmit a satellite's model to a $\mathcal {PS}$ can be longer than the satellites' visibility period and thus will fail. In the next section, we show that by using NOMA in FL-LEO, we can essentially increase the value of $R$ and thereby allow for exchanging large satellite model parameters within a short visible window.

\section{NomaFedHAP Communication Framework} \label{sec:Comm-Framework}

NomaFedHAP is a synchronous FL framework proposed to address the slow convergence of FL-LEO training due to the short and irregular visibility of LEO satellites. Fig.~\ref{pattern} shows an example of a visibility pattern for an LEO satellite constellation, which indicates that each satellite's visible window is only a few minutes and the invisible periods (i.e., gaps in between) are much longer and highly irregular. 
\begin{figure}[t]
     \centering
    {{\includegraphics[width=1\linewidth]{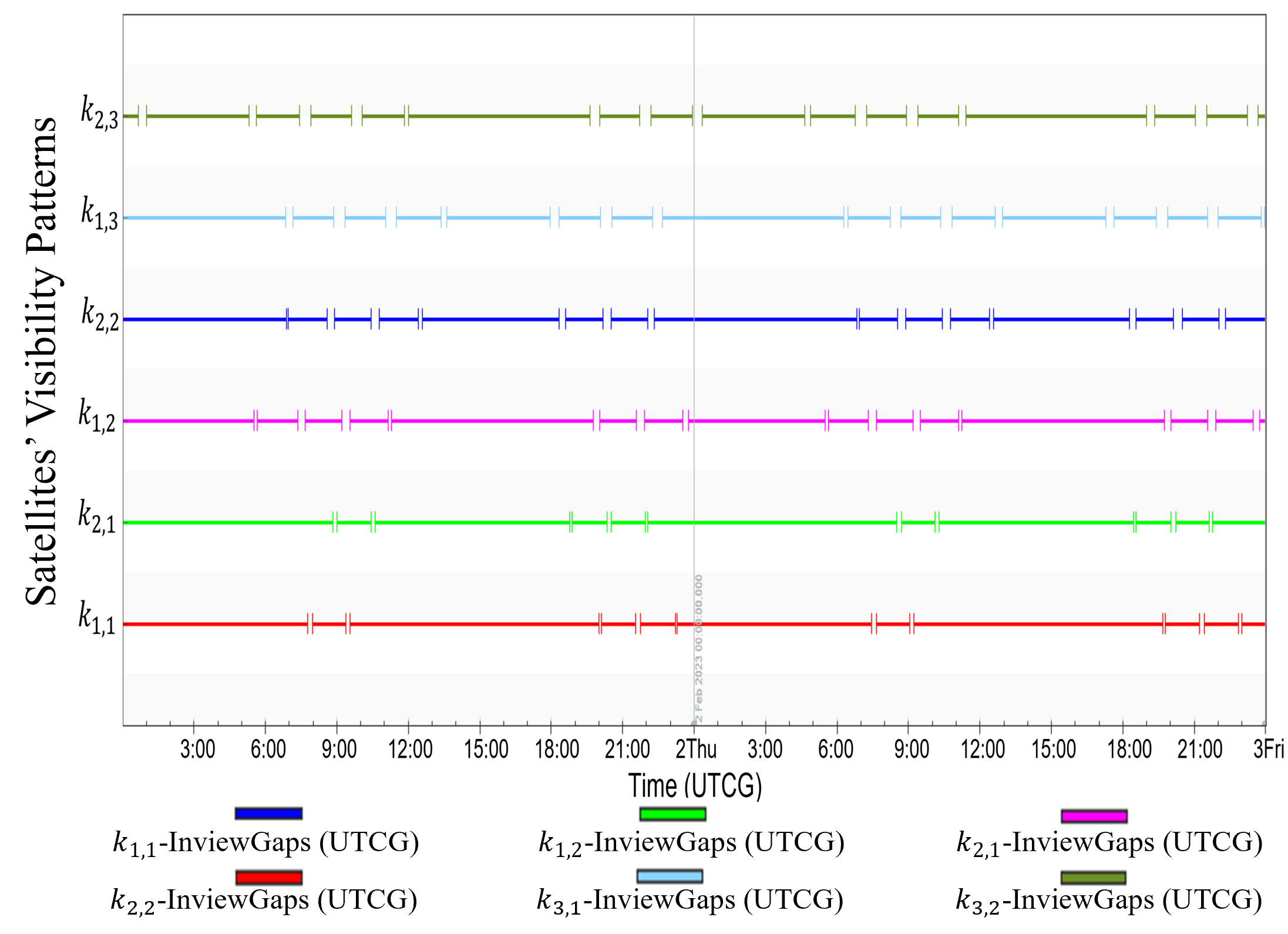}}}
    \caption{A simulated visibility pattern over two days of six LEO satellites located in three different shells, each containing two satellites positioned at an altitude of 500km, 1250km, and 2000km, respectively. The $\mathcal {PS}$ is a GS located in Rolla, Missouri, USA. In the figure, $k_{i,j}$ represents the \textit{i-th} satellite in the \textit{j-th} orbit.}
    \label{pattern}\vspace{-0.2cm}
\end{figure}

The cause of this problem is that LEO satellites fly very fast, typically taking only 90-120 minutes to orbit the Earth, while the Earth takes 24 hours to rotate one cycle. More importantly, satellites and the Earth are moving along distinct trajectories. As a result, each LEO satellite meets up with the $\mathcal {PS}$ in a very infrequent and transient manner, and eventually, this impedes the convergence of FL tremendously.

NomaFedHAP introduces the NOMA communication scheme to FL-LEO to address this issue. It enables satellites to fully utilize the entire bandwidth of downlink and exchange their ML models with the $\mathcal {PS}$ at high data rates and low bit error rates (BER), thereby drastically reducing the model transmission time to only a few seconds which is shorter than any visible window. In consequence, the issue of having straggler satellites is no longer eminent in NomaFedHAP despite its synchronous nature.

On top of that, NomaFedHAP introduces other techniques (OFDM, HAPs, model propagation, unbiased model aggregation) to expedite FL convergence further, which we describe in later sections.

\subsection{NomaFedHAP Communication Architecture}

To address the high propagation and transmission delays between LEO satellites and the traditional $\mathcal {PS}$ such as GS, LEO satellite, or MEO satellite, NomaFedHAP utilizes multiple HAPs following our work in \cite{happaper} as $\mathcal {PS}$ to propagate the local and global models between LEO satellites and HAPs. HAPs also serve as relays among orbits, mitigating the Doppler shift when the $\mathcal{PS}$ is an LEO or MEO satellite \cite{razmi, wu2022dsfl}. In the stratosphere at an altitude 17-22 km above the Earth's surface \cite{alsharoa2020improvement}, HAPs, such as unmanned airships, aircraft, or balloons, serve as quasi-stationary aerial stations that improve the network connectivity and throughput \cite{jia2020joint}. 
In general, HAPs can offer the following advantages over traditional $\mathcal {PS}$ (e.g., GS):
\begin{itemize} 
    \item {\bf Enhanced visibility:} Due to HAP's elevated altitude, it can ``see'' more satellites at once or see each satellite more frequently than GS (GS has an angular view of $180^o -  \Theta$, while a HAP can view even beyond $180^o$).
    
     \item {\bf Improved communication environment:} HAPs operate in the stratosphere which offers a clearer, stabler, and less interfered environment than the troposphere. Moreover, HAPs and LEO satellites can use FSO rather than RF links and thereby achieve a much higher data rate and lower latency (1-2 ms) \cite{xing2021high,hsieh2020uav}. Note, however, that we do {\em not} use FSO in our experiments, for a fair comparison with other approaches.
    
    \item  {\bf Lower-cost and flexible deployment:} A GS can cost millions of dollars while a HAP costs much lower \cite{hap,kurt2021vision}. Also, a GS is difficult to relocate, while a HAP can move easily for provisioning on-demand services or adapting to LEO changes.

    \item {\bf Better energy management:} HAPs can be powered by solar panels more effectively due to the higher altitude, and even completely self-powered with careful trajectory optimization \cite{marriott2020trajectory}.
\end{itemize}

Traditional FL communication follows a star topology where the $\mathcal {PS}$ sits in the center. In our work, due to the introduction of collaborative HAPs as $\mathcal {PS}$ and relays between orbits, we design a two-layer communication architecture. The first layer is a HAP layer or \textit{server layer}, which is composed of all the HAPs that aggregate and transmit global models. 
The second layer is a satellite layer or \textit{worker layer}, which is composed of all LEO satellites that train and transmit local models. We use intra-orbit ISL only and no inter-orbit ISL for communication among satellites,\footnote{A satellite has four antennas, two on the \textit{roll axis} for intra-orbit ISL communication and two on the \textit{pitch axis} for inter-orbit ISL communication.} to avoid any considerable Doppler shift. Therefore, the HAPs will serve as relays that bridge different orbits.

The server layer is organized in a {\em ring structure} for inter-HAP communication. In addition, each HAP communicates with a collection of visible satellites from different orbits as in a star topology. Therefore, the eventual communication architecture becomes a \textit{ring of small stars}.

Such a parallel connectivity pattern among the rings can significantly enhance communication. Even when there is only a single HAP, our local model propagation algorithm (\sref{Model_propagation}) allows satellites to leverage current or soon-to-be visible satellite to exchange models with $\mathcal {PS}$ without waiting for their own visible windows, thus reaping substantial performance gains.

\subsection{Introducing NOMA to FL-LEO}\label{sec:noma}

Fig.~\ref{Noma_sys} gives an overview. For illustration purposes, it only shows one orbit per shell and a single HAP. However, NomaFedHAP supports multiple orbits per shell and multiple HAPs. 

NomaFedHAP introduces a hybrid NOMA-OFDM communication scheme to LEO satellites. Specifically, visible satellites on different shells communicate with HAPs using PD-NOMA, while satellites on the same shell (i.e., at the same distance from the HAPs) communicate with HAPs and other satellites in the same orbit using OFDM (the intra-orbit communication is for the purpose of model propagation which is described in \sref{Model_propagation}).



\begin{figure*}[!t]
     \centering
     \includegraphics[width=0.9\linewidth]{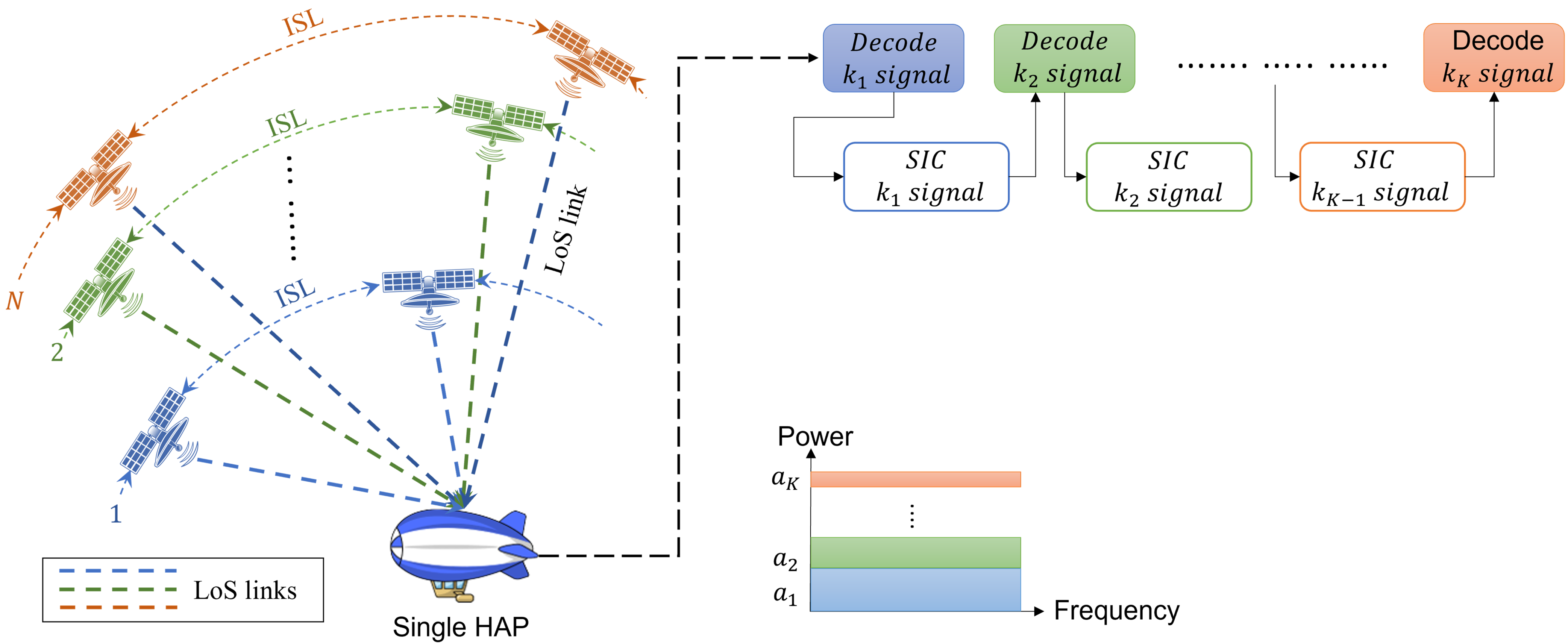}
    \caption{NomaFedHAP with orbital shells 1,2,...$N$ and a single HAP as $\mathcal {PS}$ for illustration. Only visible satellites are drawn so all the shown links are LoS links.}
    \label{Noma_sys}\vspace{-0.2cm}
\end{figure*}

Over the downlink, all the visible satellites from different shells transmit signals (i.e., their ML models) to their respective visible HAPs within the whole bandwidth using NOMA. Each HAP $\mathpzc{h}\in \mathcal H$ will thus receive a {\em combined} signal $y$ from all the satellites $\mathcal K'$ in $\mathpzc{h}$'s visible zone, which can be expressed as, due to different power allocation coefficients,
\begin{equation}
    y=n+\sum_{k\in \mathcal K^{'}} \lambda_{k} \sqrt{a_k~P_s}~x_k
    \label{receiver_eqn}
\end{equation}
where $\lambda_{k}$ is the channel coefficient of \textit{k-th} satellite, 
$a_k$ is a fractional power coefficient, $P_s$ is the maximum downlink transmission power of each visible satellite, and $x_k$ is the \textit{k-th} satellite's signal with unit energy. The other term, $n\sim \mathcal{CN}(0,\sigma^2)$, is complex additive white Gaussian noise (AWGN) with variance $\sigma^2=K_BTB$, where $K_{B} =1.38\times 10^{-23} J/K$ is the Boltzmann constant, \textit{T} is the noise temperature, and $B$ is the bandwidth.


The power coefficient $a_k$ is adjusted by each satellite $k$ based on its channel condition, and satisfies $\sum_{k\in \mathcal{K}^{'}} {a_k} \leq 1$ to limit the interference from other satellites. Note that $a_k$ is inversely related to the satellite's channel condition, which means that a satellite with a poor channel will select a higher transmission power coefficient. 

Next, each HAP starts to decode the combined signal using {\em successive interference cancellation} (SIC) iteratively. The satellites are ordered according to their channel gain (strongest first), as 
\begin{equation}
|\lambda_1|^2 \geq |\lambda_2|^2 \geq ... \geq |\lambda_{K^{'}}|^2
\end{equation}
The HAP will then decode the signal of the strongest satellite first, i.e., $x_1$, by treating the signal from other satellites as interference. Next, it re-modulates and subtracts the decoded signal $x_1$ from the received signal $y$. After that, it performs SIC for the next strongest satellite (i.e., $x_2$) and so forth until the last satellite's signal $x_{k}$ is decoded. See Fig.~\ref{Noma_sys} for an illustration.

In our analysis below, we focus on the downlink NOMA, as the uplink in our system can use any standard satellite communication scheme since it only involves the $\mathcal{PS}$ broadcasting the same global model to all the satellites, and thus each satellite independently decodes a single signal.

{\bf 1) Signal-to-Interference Noise Ratio Analysis.} 
Once all the visible satellites' signals are decoded at the HAP, the signal-to-interference noise ratio (SINR) of the first satellite (i.e., strongest signal) can be computed by \cite{aldababsa2018tutorial}:\vspace{-0.1cm}
\begin{equation} \label{eqn_15}
    SINR_{1}= a_1 \rho |\lambda_1|^2
\end{equation}
where $\rho=P_s/\sigma^2$ is the signal-to-noise ratio (SNR). 
For the the remaining visible satellites $k\in\mathcal K^{'}$, $k \neq 1$, the SINR can be calculated by
\begin{equation}\label{eqn_14}
    SINR_{k}=\frac{a_k \rho |\lambda_{k}|^2}{\rho \sum_{i=1}^{k-1} |\lambda_i|^2 a_i + 1}
\end{equation}

{\bf 2) Data Rate Analysis.}
Once the SINR of downlink NOMA is determined, we can obtain the maximum achievable rate at a HAP $\mathpzc{h}$. Assuming that the symbols $x_1, x_2,\dots, x_{k-1}$ have been decoded correctly, the maximum data rate at $\mathpzc{h}$ to decode satellite $k$'s symbol $x_k$ is given by
\begin{equation}
    R_{k\rightarrow \mathpzc{h}}={log}_ {_2}\biggl(1+ \underbrace{\frac{a_k \rho |\lambda_{k}|^2}{\rho \sum_{i=1}^{k-1} |\lambda_i|^2a_i+1}}_{SINR_{k}}\biggl).
\end{equation}
Note that $R_{k\rightarrow \mathpzc{h}}$ is normalized per unit bandwidth. Consequently, the maximum total rate $R_{total}$ for a HAP to correctly decode all its visible satellites' symbols is 
\begin{align}
      R_{total} =& \sum_{k\in \mathcal K^{'}} {log}_ {_2} (1+SINR_k)\notag\\
      =&\,\Resize{7.8cm}{ {log}_ {_2}\Big(1+ a_1 \rho |\lambda_1|^2\Big) + \sum_{k\in \mathcal K^{'},k\neq 1} {log}_ {_2} \biggl(1+ \frac{a_k \rho |\lambda_{k}|^2}{\rho \sum_{i=1}^{k-1} |\lambda_i|^2 a_i + 1}\biggl)}\notag\\
      =&\, {{log}_ {_2}\Biggl(\biggl(1+ a_1 \rho |\lambda_1|^2\biggl) \times \biggl(1+ \frac{a_2 \rho |\lambda_{2}|^2}{a_1\rho |\lambda_1|^2  + 1}\biggl)}\notag\\
      &\ \times \biggl(1+ \frac{a_3 \rho |\lambda_{3}|^2}{a_1\rho |\lambda_1|^2 +a_2\rho |\lambda_2|^2 + 1}\biggl) \times \dots\Biggl)\notag\\
      =&\, {log}_ {_2}\Big(1+ \rho \sum_{k\in \mathcal K^{'}} |\lambda_{k}|^2 a_k\Big).
\end{align}
When $\rho \gg 1$, $R_{total}$ can be approximated as\vspace{-0.1cm}
\begin{align}
      R_{total}\approx log_ {_2} \Big(\rho \sum_{k\in\mathcal K^{'}}|\lambda_{k}|^2a_k\Big).
\end{align}

{\bf 3) Channel Model Statistics.}
The downlink between each visible satellite $k$ and its connected HAP can be modeled by a shadowed-Rician fading channel, whose probability density function (PDF) of $|\lambda_{k}|^2$ is given by \cite{bankey2021physical} 
\begin{equation}\label{pdf}
    f_{|\lambda_{k}|^2}(x)=\mu_k e^{(-\beta_k x)}{_1 F_1}(m_k;1;\delta_k x)
\end{equation}
where $\mu_k=\frac{1}{2b_k}\big(\frac{2b_k  m_k}{2b_k  m_k+\Omega_k})^{m_k}$, $\beta_k=\frac{1}{2b_k}$, $\delta_k=\frac{\Omega_k}{2b_k(2b_k m_k+\Omega_k )}$, ${_1F_1(.;.;.)}$ is a {\em confluent hypergeometric function of the first type} \cite{VANASSCHE}, 2$b_k$ is the multi-path component, $m_k$ is the integer-valued fading severity parameter of the channel, and $\Omega_k$ is the average power of the LoS link. Using $m_k$, the hypergeometric function $_1F_1$ can be expressed as\vspace{-0.1cm}
\begin{align}
_1F_1(m_k;1;\delta_k x) &= e^{\delta_k x} \sum_{i=0}^{m_k-1} \underbrace{\frac{{(-1)}^i {(1-m_k)}_i {(\delta_k x)}^i}{{(i!)}^2}}_{\kappa(i)}
\nonumber\\&= e^{\delta_k x} \sum_{i=0}^{m_k-1} \kappa(i),
\end{align}
where $(\cdot)_i$ is the pochhammer symbol. With the aid of [\citenum{zwillinger2007table}, Eq. (3.351.2)], we obtain the  cumulative distribution function (CDF) for $f_{|\lambda_{k}|^2}(x)$ as in \eqref{pdf} as\vspace{-0.1cm}
\begin{align}\label{eqn_Fk}
F_{|\lambda_{k}|^2}(x) &= 1-\mu_k  e^{-(\beta_k-\delta_k)x} \sum_{i=0}^{m_k-1} \kappa(i)  \sum_{j=0}^i \frac{i!}{j!} x^j  \times \nonumber\\& (\beta_k-\delta_k )^{-(i-j+1)}.
\end{align}
It is assumed that the links between each HAP $\mathpzc{h}$ and the GS (for transmitting the final global model after the training completes) follow Nakagami-m fading, whose PDF can be expressed by \cite{gamal2022performance}\vspace{-0.1cm}
\begin{equation}
    f_{|\lambda_\mathpzc{h}|^2}(x)=\Big(\frac{m_\mathpzc{h}}{\Omega_\mathpzc{h}}\Big)^{m_\mathpzc{h}}~\frac{x^{m_\mathpzc{h}-1}}{\Gamma(m_\mathpzc{h})}~e^{-\frac{m_\mathpzc{h}}{\Omega_\mathpzc{h}}x}
\end{equation}
where $m_\mathpzc{h}$ and $\Omega_\mathpzc{h}$ are the severity parameter and the average power of LoS, respectively, for a HAP $\mathpzc h$. Hence, the CDF for $f_{|\lambda_\mathpzc{h}|^2}(x)$ can be given by\vspace{-0.2cm}
\begin{equation}\label{eqn_Fi}
  F_{|\lambda_\mathpzc{h}|^2} (x)=1-e^{-\frac{m_\mathpzc{h}}{\Omega_\mathpzc{h}}x} \sum_{n=0}^{m_\mathpzc{h}-1}\big(\frac{m_\mathpzc{h}}{\Omega_\mathpzc{h}}x\Big)  \frac{1}{n!}. 
\end{equation}
{\bf 4) Outage Probability Analysis.} Here we analyze the NOMA downlink \textit{reliability} from the perspective of system outage probability (OP). In the context of LEO satellites, OP refers to the probability that the received power at a HAP $\mathpzc{h}$ falls below a threshold such that the SINR is too low for the HAP to decode the correct signal $x_k$. Mathematically, 
\begin{equation}\label{OP_Sat}
    { OP_{\mathpzc{h}}=1-Pr\bigl( Q_1 \cap  Q_2 \cap\dots \cap Q_k \bigl)}
\end{equation}
where $Q_j, j=1,...,k$, denotes the event that a satellite signal $x_j$ is correctly decoded by HAP $\mathpzc{h}$. The OP can be rewritten as
\begin{align}
OP_{\mathpzc{h}}^{k}&=Pr \Big(SINR_{k\rightarrow \mathpzc{h}} < \gamma_{th}^k \Big) 
=\Resize{3.8cm}{Pr\bigg(\frac{a_k \rho  |{ {\lambda_{k}}}|^2}{\rho \sum_{i=1}^{k-1} | {\lambda_i}|^2a_i+1} < \gamma_{th}^k \bigg)} \notag \\
        &=Pr\Biggl(|{\lambda_{k}}|^2 < \frac{\gamma_{th}^k\bigl(\rho \sum_{i=1}^{k-1} | {\lambda_i}|^2a_i+1\bigl)}{a_k \rho }\Biggl) \notag \\
         &=Pr\Big(| {\lambda_{k}}|^2 < \eta^*_k\Big) = F_{| {\lambda_{k}}|^2 }(\eta^*_k) \notag \\ 
        &=1-\mu_k  e^{-(\beta_k-\delta_k)\eta^*_k} \sum_{i=0}^{m_k-1} \kappa(i)\sum_{j=0}^i \frac{i!}{j!} (\eta^*_k)^j\nonumber\\& \times(\beta_k-\delta_k )^{-(i-j+1)}   
\end{align}
where $\gamma_{th}^k$ is the SINR threshold of \textit{k-th} satellite to be correctly decoded and $\eta^*_k=\max\{\eta_1,\eta_2,\dots,\eta_k\}$ with $\eta_j=\frac{\gamma_{th}^j\bigl(\rho \sum_{i=1}^{j-1} | {\lambda_i}|^2a_i+1\bigl)}{a_j\rho}$. 

Below, we derive a closed-form OP expression for the nearest satellite (NS) and that for the farthest satellite (FS), which represent the strongest and the weakest signals, respectively, with respect to any HAP $\mathpzc h$. We also derive the closed-form OP for the entire LEO system subsequently.
\paragraph{\bf Derivation of OP for the NS} The outage for the NS scenario occurs when the NS' transmitted signal $x_{NS}$ cannot be successfully decoded by its connected HAP $\mathpzc{h}$, which can be expressed as\vspace{-0.09cm}
\begin{equation}\label{OP_n}
OP_{\mathpzc {h}}^{NS}= Pr\bigl(\gamma_{x}^{NS} < \gamma_{th}^{NS} \bigl)
= 1 - Pr\bigl( \gamma_{x}^{NS} \geq \gamma_{th}^{NS} \bigl)
\end{equation}
where $\gamma_{th}^{NS}=2^{2R_{NS}}-1$ is the target SINR threshold for the NS to be correctly decoded by $\mathpzc {h}$, and $R_{NS}$ is the target data rate for correctly receiving the NS's signal by $\mathpzc {h}$. The SINR $\gamma_{x}^{NS}$ can be written as\vspace{-0.1cm}
\begin{equation}\label{OP_N}
     \gamma_{x}^{NS}={{a_{NS}\rho}~|\lambda_{NS}|^2}
\end{equation}
Using \eqref{OP_N}, we can write $OP_{\mathpzc {h}}^{NS}$ as \vspace{-0.3cm}
\begin{align}
OP_{\mathpzc {h}}^{NS} &=1-Pr\biggl( |\lambda_{NS}|^2 \geq \frac{\gamma_{th}^{NS}~\omega_1}{a_{NS}} \biggl)\nonumber\\&
= 1 - \Big( 1-F_{|\lambda_{NS}|^2}\bigl(A~\omega_1\bigl)\Big)=\Resize{2.15cm}{F_{|\lambda_{NS}|^2}\bigl(A~\omega_1\bigl) }
\end{align}
where $A=\frac{\gamma_{th}^{NS}}{a_{NS}}$ and $\omega_1=\frac{1}{\rho}$. With the aid of \eqref{eqn_Fk}, we obtain the closed-form expression for $OP_{\mathpzc{h}}^{NS}$ as
\begin{align}\label{OP_NS}
        OP_{\mathpzc {h}}^{NS} &=1-\mu_k  e^{-(\beta_k-\delta_k)A\omega_1} \sum_{i=0}^{m_k-1} \kappa(i)\sum_{j=0}^i \frac{i!}{j!} (A \omega_1)^j \nonumber\\&  \times(\beta_k-\delta_k )^{-(i-j+1)} 
\end{align}
\paragraph{\bf Derivation of the outage for the FS} The OP for the FS scenario occurs under two conditions: i) its connected HAP $\mathpzc{h}$ fails to decode both the NS signal $x_{NS}$ and its transmitted signal $F_{NS}$, and ii) $\mathpzc{h}$ can decode the NS signal $x_{NS}$ but cannot decode its signal $x_{FS}$. Mathematically, that is 
\begin{align}\label{OP_F}
   OP_{\mathpzc {h}}^{FS}=&Pr\bigl(\gamma_{x}^{NS}<\gamma_{th}^{NS} \bigl)  Pr\bigl(\gamma_{x}^{FS}<\gamma_{th}^{FS} \bigl) +  \nonumber\\& Pr\bigl(\gamma_{x}^{NS}\geq\gamma_{th}^{NS}\bigl)Pr\bigl(\gamma_{x}^{FS}<\gamma_{th}^{FS} \bigl)\nonumber\\
        =&1-Pr\bigl(\gamma_{x}^{FS}\geq\gamma_{th}^{FS}\bigl) 
\end{align}
Similar to NS,  $\gamma_{th}^{FS}=2^{2R_{FS}}-1$  denotes the target SINR threshold of the FS to be correctly decoded by $\mathpzc {h}$, and $R_{FS}$ is the target data rate for correctly receiving the FS's signal by $\mathpzc {h}$. But here  $\gamma_{x}^{FS}$ is given by
\begin{equation}\label{OP_N_F}
\begin{aligned} 
    &\gamma_{x}^{FS}=\frac{a_{FS}\rho|\lambda_{FS}|^2}{\rho \sum_{i=1}^{FS-1}|\lambda_{i}|^2 a_{i}+1}
\end{aligned}
\end{equation}
Similar to the derivation of $OP_{\mathpzc {h}}^{NS}$, we can derive the closed-form expression for the OP of the FS scenario by substituting from \eqref{eqn_Fk} as follows:\vspace{-0.2cm}
\begin{align}\label{OP_FS}
        OP_{\mathpzc {h}}^{FS}&=1-\mu_k  e^{-(\beta_k-\delta_k)E\omega_2} \sum_{i=0}^{m_k-1} \kappa(i)\sum_{j=0}^i \frac{i!}{j!} (E \omega_2)^j \nonumber\\& \times(\beta_k-\delta_k )^{-(i-j+1)}
\end{align}
where $E=\frac{\gamma_{th}^{FS}}{a_{FS}}$ and $\omega_2=\frac{\rho \sum_{i=1}^{FS-1}|\lambda_{i}|^2 a_{i}+1}{\rho}$.\vspace{1mm}

\paragraph{\bf Derivation of OP for the entire LEO system} By combining the outage experience at $\mathpzc {h}$ for both NS and FS scenarios, the overall system OP can be expressed as\footnote{Note that our derivation has accounted for the case when there are extra satellites between NS and FS, which can be seen from \eqref{OP_N_F}.}
\begin{align}\label{OP_sys}
     OP_{sys}=&1-Pr\bigl(\gamma_{x}^{NS}\geq\gamma_{th}^{NS} \bigl) Pr\bigl(\gamma_{x}^{FS}\geq\gamma_{th}^{FS}\bigl)\nonumber
        \\=&1-\Big( 1-F_{|\lambda_{NS}|^2}\bigl(A~\omega_1\bigl)\Big)\times \Big( 1-F_{|\lambda_{FS}|^2}\bigl(E~\omega_2\bigl)\Big)\nonumber
        \\ =&1-\Big(  \mu_k  e^{-(\beta_k-\delta_k)A\omega_1} \sum_{i=0}^{m_k-1} \kappa(i)\sum_{j=0}^i \frac{i!}{j!} (A \omega_1)^j \nonumber
         \\&   \times(\beta_k-\delta_k )^{-(i-j+1)}  \Big)  \times\Bigl(\mu_k  e^{-(\beta_k-\delta_k)E\omega_2}\sum_{i=0}^{m_k-1} \kappa(i) \nonumber
         \\& {\sum_{j=0}^i \frac{i!}{j!} (E \omega_2)^j  \times(\beta_k-\delta_k )^{-(i-j+1)} \Bigl)}
\end{align}

To the best of our knowledge, the OP of NOMA  for LEO satellites as ``clients'' has never been derived in the literature. Additionally, we also demonstrate via simulations that our NomaFedHAP scheme achieves higher data rates and experiences less outages even when a large number of satellites communicate simultaneously with the $\mathcal{PS}$.  
\begin{figure*}[t]
     \centering
         \subfloat[\centering Global model propagation within the HAP layer.]{\centering 
         \includegraphics[width=0.32\textwidth]{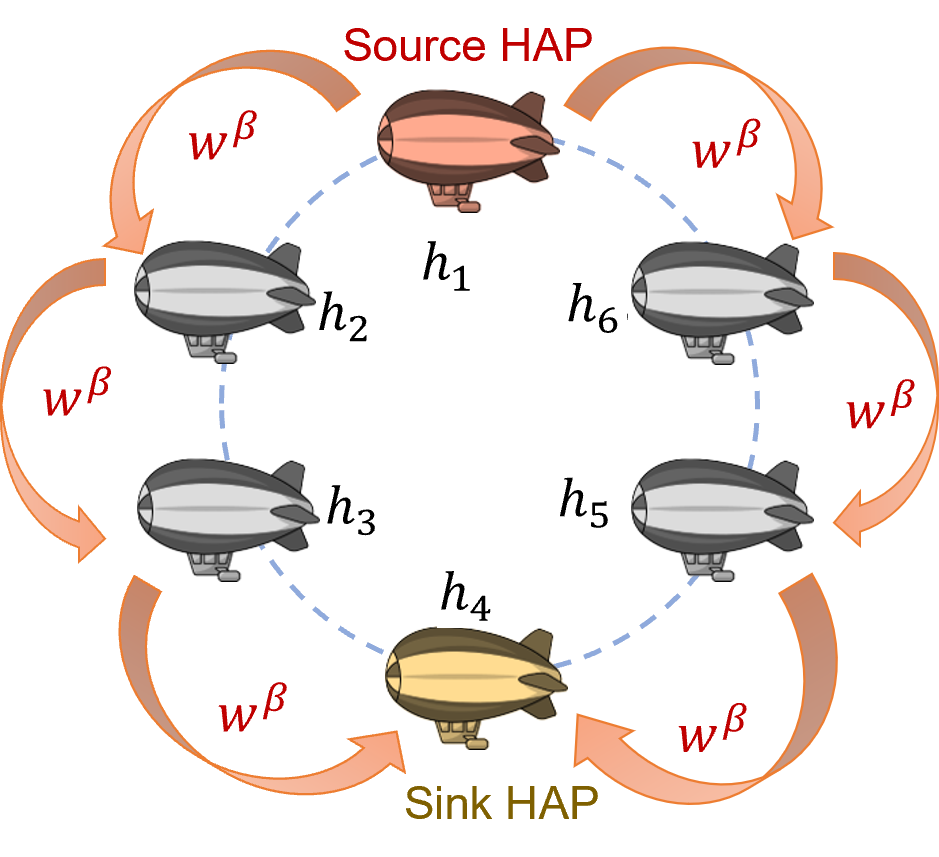}\label{Picture2}}
     \hfill
     \subfloat[\centering Local and sub-orbital models propagation within the Satellite Layer.]{\centering 
         \includegraphics[width=0.32\textwidth]{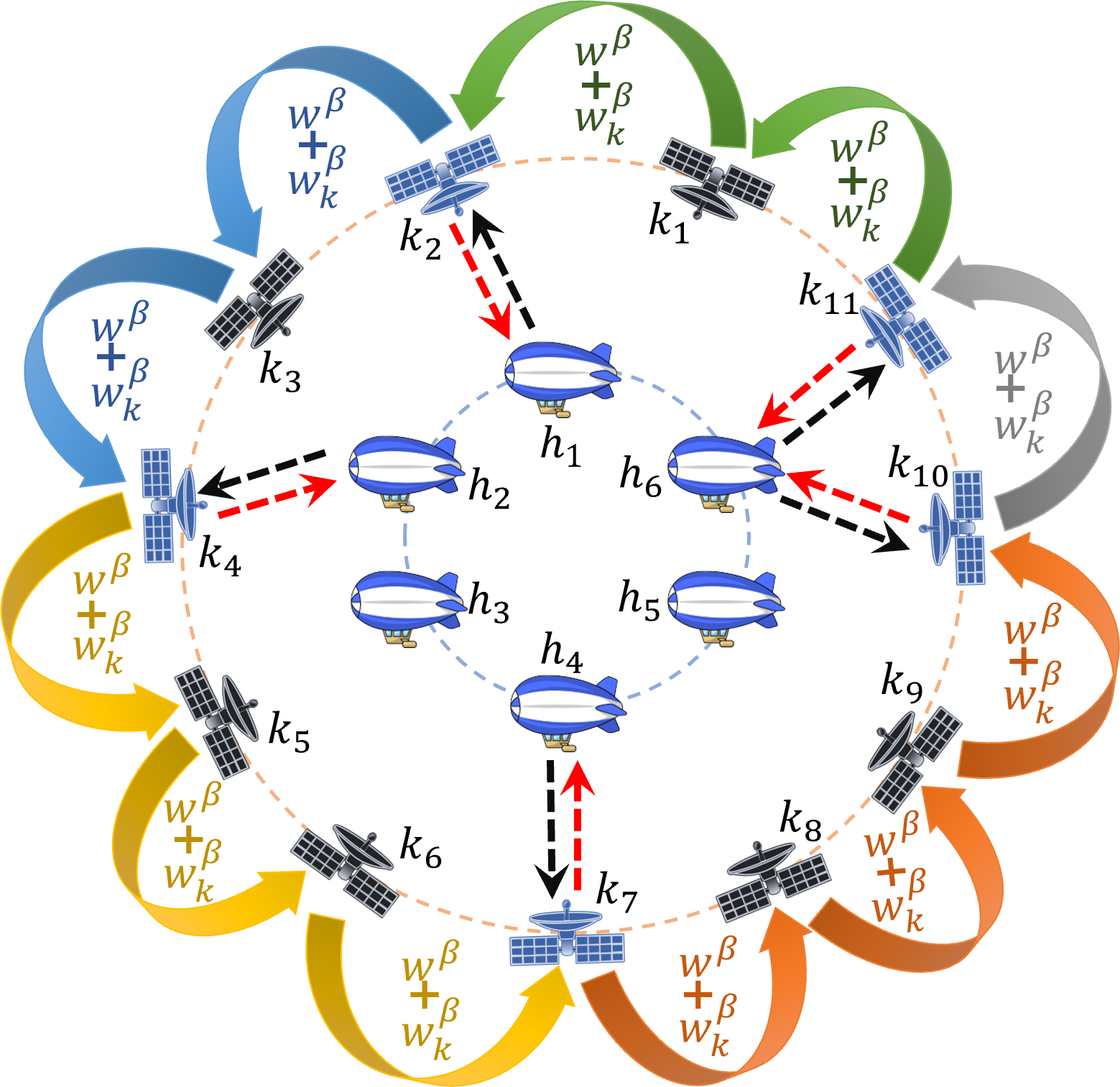}\label{Picture4}}  
     \hfill
     \subfloat[\centering Sub-orbital model propagation within the HAP layer.]{\centering 
         \includegraphics[width=0.32\textwidth]{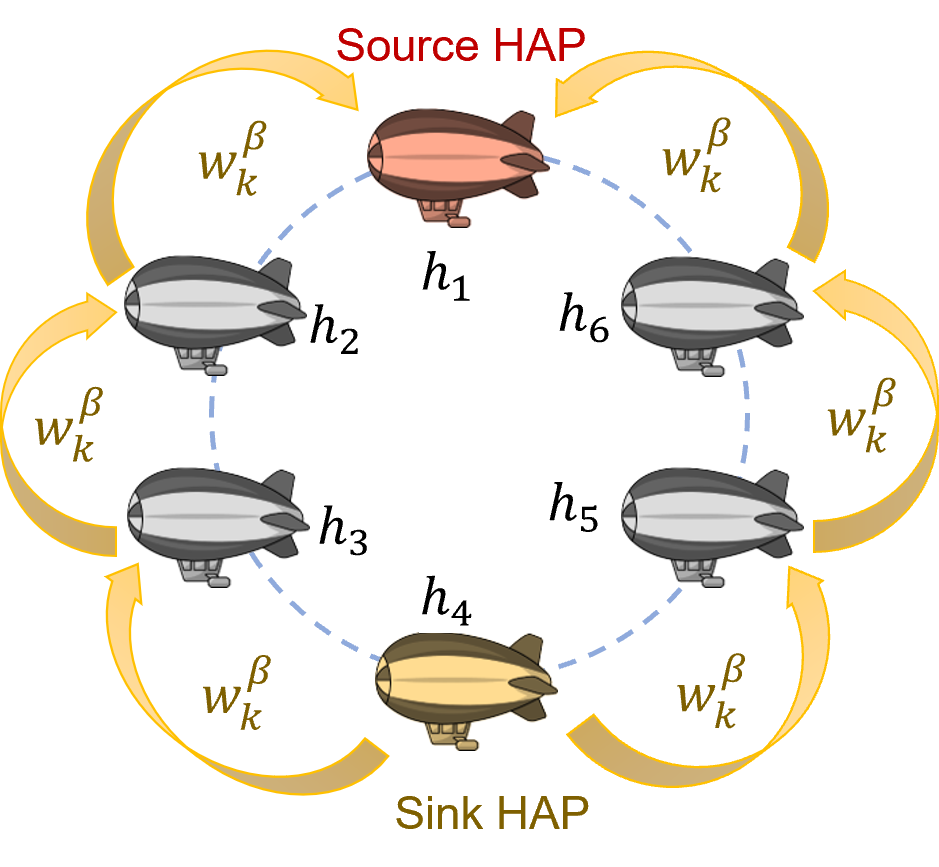}\label{Picture3} }        
\hfill
        \caption{Illustration of the proposed model propagation algorithm. (a) The source HAP $\mathpzc{h}_1$ forwards the global model $\boldsymbol{w}^\beta$ to the sink HAP $\mathpzc{h}_4$; (b) Visible satellites are represented by blue, and invisible satellites by black. Colored curved arrows show the propagation of models (global model and sub-orbital model) from $k_2 \rightarrow k_4$, $k_4 \rightarrow k_7$, $k_7 \rightarrow k_{10}$, $k_{10} \rightarrow k_{11}$, $k_{11} \rightarrow k_{2}$. (c) The sink HAP $\mathpzc{h}_4$ propagates the sub-orbital model $w_k{^\beta}$ to the source HAP $\mathpzc{h}_1$ along the reverse pathway.}\label{fig:propagate}
\end{figure*}

\section{NomaFedHAP Convergence Framework}\label{sec:Model_aggregation}

In NomaFedHAP, FL convergence is achieved through two key components, a {\em model propagation} algorithm, and a {\em model aggregation} algorithm. The model propagation algorithm is proposed to minimize the ``idleness'' in traditional synchronous FL-LEO approaches, where ``straggler'' satellites cause the $\mathcal{PS}$ to idly wait for long periods for model exchange. Note that unlike existing works which resort to asynchronous FL and thereby face the {\em stale model} problem, our solution keeps the synchronous nature (and hence benefits from all instead of a subset of models) yet still accelerates the process substantially, by enabling intra-orbit client communications.
The second component, the model aggregation algorithm, runs on the HAP layer and aggregates the sub-orbital models received from each HAP. This model aggregation algorithm differs from traditional FL, which only aggregates {\em individual} client models.
\subsection{NomaFedHAP Model Propagation Algorithm}\label{Model_propagation}

This algorithm consists of propagation of global, local, and sub-orbital models, as illustrated in \fref{fig:propagate} and explained below.

{\bf Global Model Propagation within HAP Layer} (\fref{Picture2}). 
When there are multiple HAPs in the HAP layer, one HAP will be designated as the {\em source} while the other (the farthest one from the source) as the {\em sink}. To begin with, the source HAP generates the initial global model $\boldsymbol{w}^{0}$ and sends it to its adjacent HAPs via  IHL. Simultaneously, it also broadcasts $\boldsymbol{w}^{0}$ to its currently visible satellites at different altitudes using our proposed NOMA-OFDM scheme (see \sref{sec:noma}). The adjacent HAPs, upon receiving $\boldsymbol{w}^{0}$, forward $\boldsymbol{w}^{0}$ to their respective next-hop neighbors, and also send it to their respective visible satellites. This continues until the sink HAP receives $\boldsymbol{w}^{0}$ and transmits it to its currently visible satellites. In subsequent rounds ($\beta=1,2,...$), the same procedure as above takes place again, except that $\boldsymbol{w}^{0}$ is replaced by $\boldsymbol{w}^{\beta}$.


\begin{algorithm}[!t]
\caption{\small NomaFedHAP model propagation algorithm}\label{algorithm1}
\kwInit{Global iteration $\beta$=0, $\boldsymbol{w}^{\beta}$, and $\mathcal U_\mathpzc{h}|_{\mathpzc{h}\in \mathcal H}=\phi$}
\While{Termination criterion is not met }{
\ForEach(\Comment*[h]{\scriptsize Global Model Propagation from source to sink HAP}){$\mathpzc{h}\in \mathcal H$}{
  Forward $\boldsymbol{w}^{\beta}$  to its next-neighbor HAP\\
  \If{ $\mathpzc{h}$ has LoS connection with satellites}{
  Transmit $\boldsymbol{w}^{\beta}$ to its visible satellites\\}
\ForEach(\Comment*[h]{\scriptsize Local/Sub-orbital Models propagation}){$k\in \mathcal K$ that is visible to $\mathpzc{h}$ }{  
Retrain $\boldsymbol{w}^{\beta}$ to update $\boldsymbol{w}_{k}^{\beta}$ using its own data\\  

\ForEach(\Comment*[h]{\scriptsize Aggregation of sub-orbital models}){invisible $k'$ between $k$ and $k+1$}{

Retrain $\boldsymbol{w}^{\beta}$ to update $\boldsymbol{w}_{k'}^{\beta}$ using its own data\\
Aggregate $\boldsymbol{w}_{k}^{\beta}$ and $\boldsymbol{w}_{k'}^{\beta}$ using (\ref{eq:orbitalagg})\\ 
Propagate $\boldsymbol{w}^{\beta}$ and $\boldsymbol{w}_{k}^{\beta}$ to next $k'$ 
} 
Transmit  $\boldsymbol{w}_{k}^{\beta}$ by Sat $k+1$ to its visible HAP\\ 
Update $\mathcal U_{\mathpzc{h}} \leftarrow  \mathcal U_{\mathpzc{h}} \cup \{\boldsymbol{w}_{k}^{\beta}\}$\\
Store all the propagating Sat IDs\\
}
}
\ForEach (\Comment*[h]{\scriptsize Sub-orbital Models Propagation from sink to source HAP}){$\mathpzc{h}\in \mathcal H$}{ 
{Transmit} $\mathcal U_{\mathpzc{h}}$ to the next neighboring HAP
}
{$ \beta \leftarrow  \beta+1$}}
\nonl\rule{0.95\linewidth}{0.5pt}\\
\nonl{\parbox{\linewidth}{\hspace*{-1em}\scriptsize{{\bf Note:} The above process provides a  sequential representation for clarity. However, in real-world scenarios and our simulations, computation and transmission occur concurrently.}}}
\end{algorithm}

{\bf Local and Sub-orbital Model Propagation within Satellite Layer} (\fref{Picture4}).
In any round, say the $\beta$-th, as soon as the visible satellites successfully receive and decode the global model $\boldsymbol{w}^{\beta}$, each of them performs two tasks. First, it retrains $\boldsymbol{w}^{\beta}$ using its own data to obtain an updated local model $\boldsymbol{w}_{k}^{\beta}$.
Second, it transmits both $\boldsymbol{w}^{\beta}$ and a weighted version of $\boldsymbol{w}_{k}^{\beta}$, which is $\boldsymbol{w}_{k}^{\beta} := \gamma_{k} \boldsymbol{w}_{k}^{\beta}+\boldsymbol{0}$
(see Eq.~\eqref{eq:orbitalagg} below, where $\gamma_{k}$ is defined similarly), to its next-hop satellite $k'$ via ISL. The propagation direction is pre-designated as either clockwise or counterclockwise. Sending $\boldsymbol{w}^{\beta}$ is to ensure $k'$ to have a copy of the global model regardless of whether $k'$ is visible to any $\mathcal {PS}$. Next, the satellite $k'$ will first retrain $\boldsymbol{w}^{\beta}$ to obtain $\boldsymbol{w}_{k'}^{\beta}$, like what $k$ did; then, it will perform \textit{sub-orbital aggregation} by combining its own $\boldsymbol{w}_{k'}^{\beta}$ with the received $\boldsymbol{w}_{k}^{\beta}$ (which has been weighted by $k$) as follows:
\begin{align}\label{eq:orbitalagg}
    \boldsymbol{w}_{k'}^{\beta}= \gamma_{k'} \boldsymbol{w}_{k'}^{\beta}+\boldsymbol{w}_{k}^{\beta} 
\end{align}
where $\gamma_{k'} = {|\mathcal {D}_{k'}|}/{|\mathcal {D}|}$ is a scaling factor that weighs model importance according to data size, $|\mathcal {D}_{k'}|$ is the data size of satellite $k'$ and $|\mathcal {D}|$ is the sum of all the data sizes in the same orbit. Thus, $\boldsymbol{w}_{k'}^{\beta}$ is a partially aggregated model which we refer to as an {\em sub-orbital model}. Next, $\boldsymbol{w}_{k'}^{\beta}$ will be sent to the next-hop satellite (say $k''$), together with the global model $\boldsymbol{w}^{\beta}$, like above. This uni-directional forwarding continues until reaching a visible satellite (say $k^*$), which will stop forwarding further; instead, after training and partial-aggregation like above, it will transmit the aggregated model $\boldsymbol{w}_{k^*}^\beta$ to its visible HAP using NOMA, with a power coefficient based on its current altitude (static and dynamic power allocations are both evaluated in \sref{sec:Evaluation}). Hence essentially, Eq.~\eqref{eq:orbitalagg} is FedAvg computed in a {\em sequential} manner.

In summary, unlike traditional FL approaches where the $\mathcal {PS}$ must wait for all the satellites to be visible for an appropriate visibility period before receiving their updated local models and then aggregating them into a global model, we are able to ``activate'' all satellites, even those that are invisible or visible within a short visibility period (through introducing the NOMA scheme), by propagating satellite local models together with the sub-orbital models to invisible satellites within the same orbit, and thus accelerate the FL convergence processes.

{\bf Sub-orbital Model Propagation within HAP Layer} (Fig.~\ref{Picture3}).
After each HAP receives the sub-orbital models from all its visible satellites, it will propagate these sub-orbital models along the \textit{reverse} pathway (from the \textit{sink} HAP to the \textit{source} HAP). The source HAP will then aggregate all the received sub-orbital models into a global model $\boldsymbol{w}^{\beta+1}$, following Section~\ref{sec:aggreg} (Eq. \ref{eq:agg}), and propagates $\boldsymbol{w}^{\beta+1}$ to all the HAPs as in phase 1 (\fref{Picture2}).

Algorithm~\ref{algorithm1} summarizes the above three phases of model propagation. It has an overall complexity of $\mathcal{O}(\mathsf{T}*\mathsf{A}))$, where $\mathsf{T}$ is the number of iterations until the termination criterion is met, and $\mathsf{A}$ represents the nested loop operations governing the training, aggregation, and propagation of models by invisible satellites. Please note, the loops at lines 2, 6, 8, and 15 occur concurrently, while lines  9-11 run sequentially.


\begin{algorithm}[!t]
\caption{\small NomaFedHAP model aggregation algorithm}\label{algorithm2}
\kwInit{$\mathcal U^{\beta}\neq\phi$}

\While{Termination criterion is not met }{
{Sort} all received sub-orbital models as \eqref{Eqn34} and \eqref{Eqn35}\\
{Filter} redundant sub-orbital models from $\mathcal{U^{\beta}}$  according to satellite IDs \\
{Generate} $\mathcal{U'^{\beta}}$\\
\eIf{Source HAP receives all sub-orbital models}{
Aggregate $\boldsymbol{w}^{\beta+1}$ using (\ref{eq:agg})\\
}{{Wait} for all sub-orbital models to be received}
{$ \beta \leftarrow  \beta+1$}}
\end{algorithm}

\subsection{NomaFedHAP Model Aggregation Algorithm}\label{sec:aggreg} 
When all HAPs gather the sub-orbital models from their visible satellites, a propagation round begins from the sink HAP to the source HAP by forwarding the received models. Once the source HAP receives all sub-orbital models, it sorts them as follows:
\begin{equation} \label{Eqn34}
    \mathcal{U}^{\beta}= \{ S_{1}, S_{2},\dots,S_\textit{L}\}
\end{equation}
where $\mathcal{U}^{\beta}$ is the set of all sub-orbital models received by HAPs in round $\beta$, and 
 $S_{l}\subset\mathcal{U}^{\beta}$ is a subset of $\mathcal{U}^{\beta}$ that comprises all sub-orbital models for an orbit $l$, which can be expressed as
\begin{equation} \label{Eqn35}
    S_{l}= \Big\{ \underbrace{\{{\boldsymbol{w}_{k}^{\beta}}\}_{\mathpzc{h}_1}}_{\mathcal{U}_{\mathpzc{h}=1}},\underbrace{\{{\boldsymbol{w}_{k}^{\beta}}\}_{\mathpzc{h}_2}}_{\mathcal{U}_{\mathpzc{h}=2}},\dots,\underbrace{\{{\boldsymbol{w}_{k}^{\beta}}\}_{H}}_{\mathcal{U}_{H}}\Big\}_l
\end{equation}
It is possible for $S_{l}$ to encompass redundant sub-orbital models, particularly when a satellite is visible to multiple HAPs at the same time. In such cases, NomaFedHAP utilizes satellites' IDs, which are unique and sent as metadata with each sub-orbital model, to filter out these redundant models. Consequently, NomaFedHAP yields $\mathcal{U}'^{\beta} = \{S'_{l_1},S'_{l_2},\dots,S'_{L} \}$, where $S'_l$ is a set of distinct sub-orbital models for an orbit $l$.

Subsequently, for all orbits $L$, NomaFedHAP checks whether any satellite ID has been excluded from $\mathcal{U}'^{\beta}$. This scenario occurs infrequently, typically when an orbit lacks visible satellites to any HAP for an extended time. In such cases, NomaFedHAP will not generate an updated version of $\boldsymbol{w}^\beta$ immediately. Instead, it waits until any HAP $\mathpzc{h}$ receives the sub-orbital models containing the IDs of those satellites. These sub-orbital models are then transmitted to the source HAP to update $\mathcal{U}'^{\beta}$. This ensures a balanced collection of models from all orbits, allowing all satellites to contribute equally in generating $\boldsymbol{w}^\beta$. It also prevents biasing the global model toward a specific orbit.

Once the source HAP has received all the remaining sub-orbital models and updated $\mathcal{U}'^{\beta}$, NomaFedHAP aggregates all the models in $\mathcal{U}'^{\beta}$ as follows: 
\begin{equation} \label{eq:agg}
   \boldsymbol{w}^{\beta+1}= \sum_{{l=1}}^{{L}} \sum_{{\mathpzc{h}=1}}^{{H}} \frac{|\mathcal {D}|_{{\mathcal {U}_\mathpzc{h}'}}^{l}}{|\mathcal {D}|_{l}} \boldsymbol{w}_{\mathcal {U}_\mathpzc{h}'}  ^{\beta}
\end{equation}
where $|\mathcal {D}|_{\mathcal {U}_\mathpzc{h}'}^{l}$ is the total data size of the satellites in the set ${\mathcal {U}_\mathpzc{h}'}$ for an orbit $l$, whereas $|\mathcal {D}|_{l}$ is the total data size for an orbit $l$. Subsequently, the entire procedure will recommence from \sref{Model_propagation}, until the FL model is converged.

Algorithm~\ref{algorithm2} summarizes the entire process. It has an overall complexity of $\mathcal{O}(\mathsf{T}(\mathsf{B}Log_{2}\mathsf{B}))$, where $\mathsf{B}$ is the number of received sub-orbital models. The $(\mathsf{B}Log_{2}\mathsf{B})$ term signifies the complexity of sorting and organizing the models, which dominates the filtering operations and the subsequent aggregation process.

\subsection{Convergence Analysis of NomaFedHAP}
In this section, we analyze the convergence of the NomaFedHAP approach. To do that we make the following assumptions regarding the loss functions of the satellites $F_1, \ldots, F_K$, $1\leq k\leq K$. These assumptions align with the commonly encountered assumptions in the FL literature. \cite{Li2020On, ribero2022federated}.
\begin{asu}[Smoothness]\label{as:1}
All the functions $F_1, \ldots, F_K$ in Equation \eqref{eqn:loss} exhibit $\Lambda$-smoothness, as for any $\boldsymbol{a}$ and $\boldsymbol{b} \in \mathbb{R}^d$ and any $k \in \mathcal{K}$, it holds that:
$$F_k(\boldsymbol{a})\leq F_k(\boldsymbol{b})+{(\boldsymbol{a}-\boldsymbol{b})}^{\top}\nabla F_{k}(\boldsymbol{b})+\frac{\Lambda}{2}\|\boldsymbol{a}-\boldsymbol{b}\|^2_2$$
\end{asu}
\begin{asu}[Strong convex]\label{as:2}
All the functions $F_1, \ldots, F_K$ in Equation \eqref{eqn:loss} exhibit $\varrho$-strongly convex, as for any $\boldsymbol{a}$ and $\boldsymbol{b} \in \mathbb{R}^d$ and any $k \in \mathcal{K}$, it holds that:
$$F_k(\boldsymbol{a})\geq F_k(\boldsymbol{b})+{(\boldsymbol{a}-\boldsymbol{b})}^{\top}\nabla F_{k}(\boldsymbol{b})+\frac{\varrho}{2}\|\boldsymbol{a}-\boldsymbol{b}\|^2_2$$
\end{asu}
\begin{asu}[Bounded variance]\label{as:3}
Let $\xi_{k}^{\beta}$ be a data point randomly sampled from the dataset $D_k$ of satellite $k$. The variance of the stochastic gradients at each satellite is constrained as follows:
$$\mathbb{E}\|\nabla F_k(\boldsymbol{w}_{k}^\beta,\xi_{k}^{\beta})-\nabla F_{k}(\boldsymbol{w}_{k}^\beta)\|_2^2\leq \sigma_{k}^{2}$$ 
This constraint applies to all satellites, with $k$ ranging from 1 to $K$.
\end{asu}
\begin{asu}[Bounded stochastic gradients]\label{as:4}
The square norm of the expected stochastic gradients of $F_k$ is uniformly bounded, satisfying the inequality:
$$\mathbb{E}\|\nabla F_k(\boldsymbol{w}_{k}^\beta,\xi_{k}^{\beta})\|_2^2\leq G^{2}$$ 
This constraint applies to all satellites, with $k$ ranging from 1 to $K$.
\end{asu}
\begin{theorem}\label{th:Theorem}
Supposing that Assumptions 1-4 are met, with $\Lambda, \varrho, \sigma_{k},$ and $G$ defined accordingly, we can consider a Federated Learning in Low Earth Orbit (FL-LEO) configuration with $K$ fully participating satellites in each round $\beta$. In this setup, the goal is to train a machine learning model as in Equation \eqref{eqn:loss}, and as a result, the NomaFedHAP framework fulfills the following condition:
\begin{equation}
    \mathbb{E}[F(\boldsymbol{w}^{|\mathcal{B}|})]-F^*\leq\frac{2\upsilon}{\delta+|\mathcal{B}|}\biggl(\frac{Z}{\varrho}+2 \Lambda\|\boldsymbol{w}^{0}-\boldsymbol{w}^{*}\|^2_2 \biggl)
\end{equation}
Here, we set $\upsilon$ as $\frac{\Lambda}{\varrho}$, $\delta$ as $\max\{8\upsilon,J\}$, the learning rate $\zeta_{{\beta}}$ as $\frac{2}{\varrho(\delta+\beta)}$, and define $Z$ as 
$$Z=\sum_{k=1}^{K} {\alpha_k}^2{\sigma_{k}}^2+6\Lambda \Gamma+8(J-1)^2G^2$$
where $\alpha_k$ is calculated as $\frac{|\mathcal{D}_k|}{|\mathcal{D}|}$ and represents the full satellites participation, and $\Gamma=F^*-\sum_{k=1}^{K}\alpha_kF^{*}_{k}\geq0$.
\end{theorem}
We provide the proof of Theorem 1 in Appendix A.

\section{Performance Evaluation}\label{sec:Evaluation}

\subsection{Experiment Setup}

\textbf{LEO Satellite Constellation.} 
We examine a Walker-delta constellation $\mathcal K$  \cite{walker1984satellite} consisting of 60 LEO satellites in six orbits, with ten satellites in each orbit (see Figs.~\ref{Picture_6} and \ref{Picture_7}). The six orbits are located on three different shells at altitudes of 500 km, 1000 km, and 1500 km above the Earth's surface. Each shell contains two orbits and each orbit has an inclination angle of 70$^\circ$. We examine a variety of $\mathcal {PS}$ scenarios, including GS, single HAP, two HAPs, and three HAPs. For both the GS and single-HAP scenarios, they are located at Rolla, Missouri, USA (but can be anywhere of/above the Earth). The scenarios with two/three HAPs involve one HAP positioned above the city of Chinook, MT, USA, and another HAP positioned above the city of Primorsky Krai, Russia, in addition to the HAP above the city of Rolla, USA. All $\mathcal {PS}s$  are situated at an altitude of 25 km above the Earth's surface and maintain a minimum elevation angle of 10$^\circ$. To compute the visiting pattern of LEO satellites to each $\mathcal {PS}$, we use a simulator called Simulator Tool Kits (STK) developed by AGI. All $\mathcal {PS}$-Satellite connections are monitored over a period of three days to obtain a comprehensive set of results.
\begin{figure*}[t]
     \centering
    \includegraphics[width=1\linewidth]{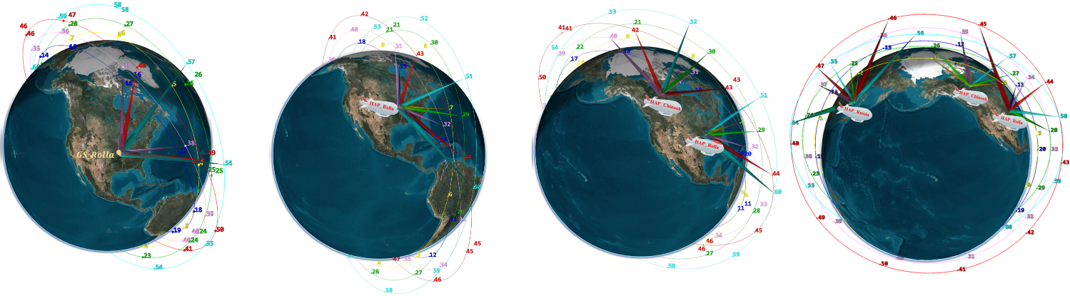}
    \subfloat[ GS located in Rolla.]{\hspace{.25\linewidth}}
    \subfloat[Single HAP]{\hspace{.25\linewidth}}
   \subfloat[Two HAPs ]{\hspace{.25\linewidth}} 
    \subfloat[Three HAPs ]{\hspace{.25\linewidth}}
      \vspace{-0.5mm}
    \caption{Simulated Walker-delta constellation in 3D with a variety of $\mathcal {PS}$ scenarios: (a) GS located in Rolla, MO, USA, (b-d) HAPs located above Rolla, USA; Chinook, USA; and Primorsky Krai, Russia.}
\label{Picture_6}
\vspace{1mm}
     \centering
     \subfloat[GS located in Rolla.]
         {\centering 
         {\includegraphics[width=0.24\linewidth]{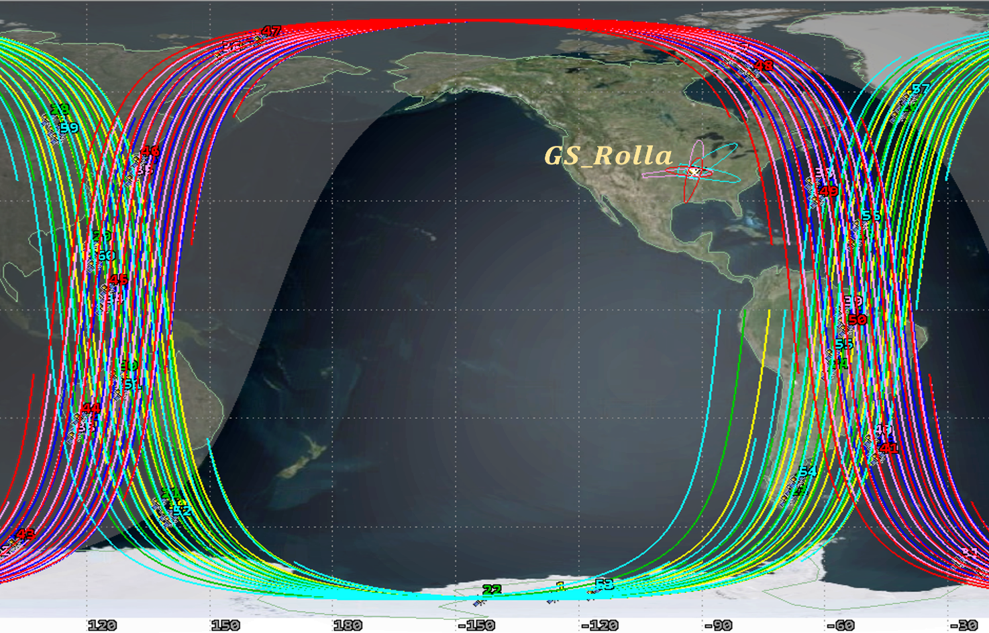}}}
     \hfill
 \subfloat[Single HAP]{
         \includegraphics[width=0.24\linewidth]{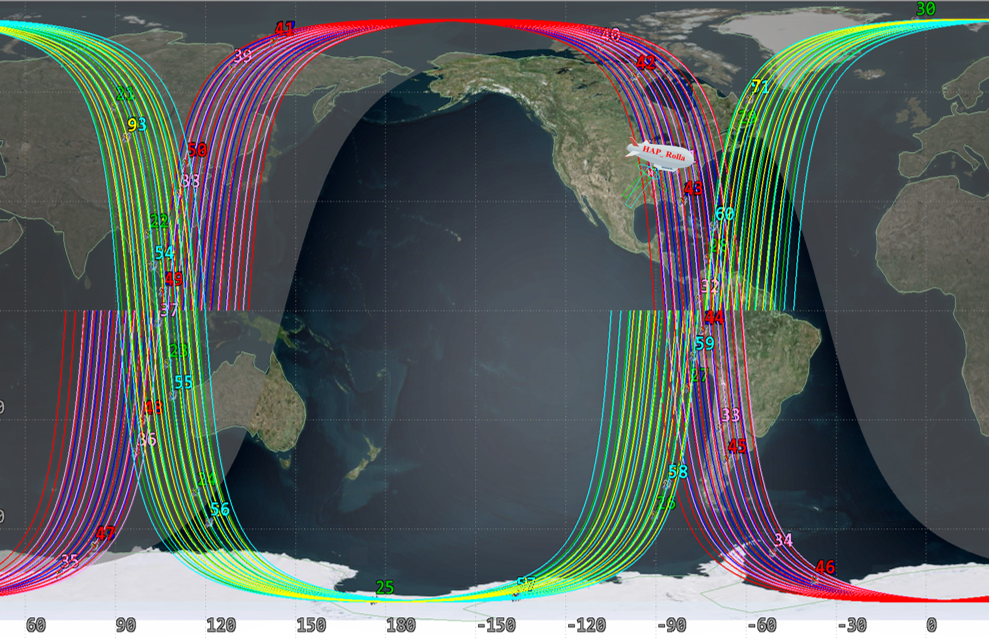}}
     \hfill
 \subfloat[Two HAPs]{
         \includegraphics[width=0.24\linewidth]{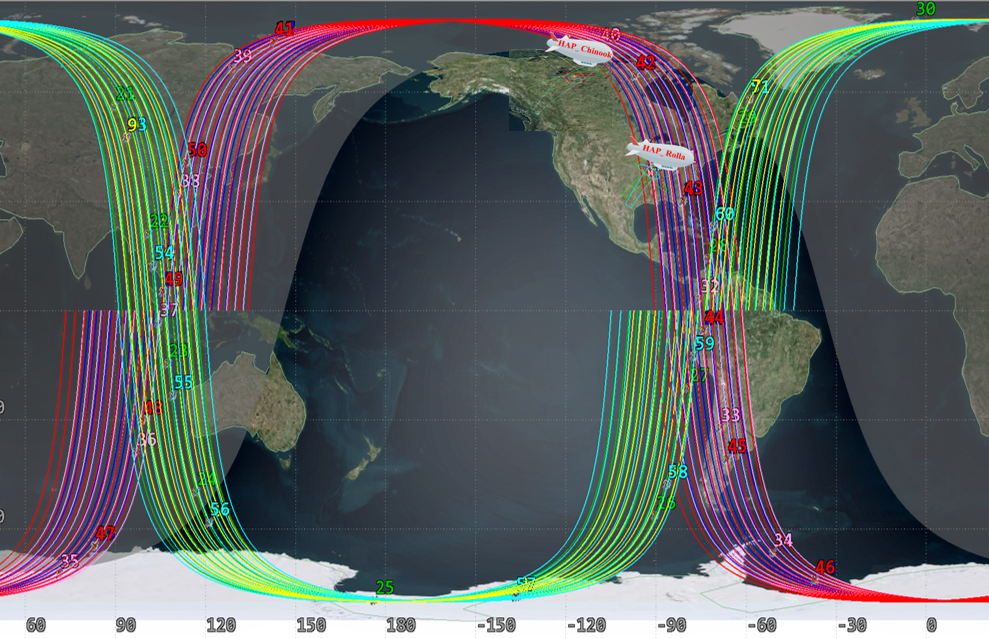}}
   \hfill
 \subfloat[Three HAPs]{
         \includegraphics[width=0.245\linewidth]{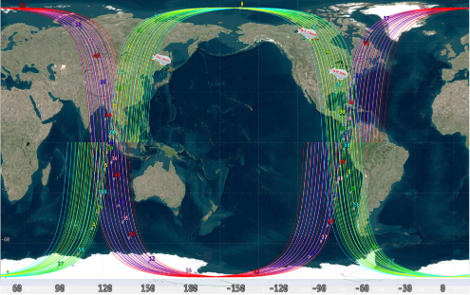}}
        \caption{Simulated Walker-delta constellation in 2D with a variety of $\mathcal {PS}$ scenarios: (a) GS located in Rolla, MO, USA, (b-d) HAPs located above Rolla, USA; Chinook, USA; and Primorsky Krai, Russia}\vspace{-3mm}
\label{Picture_7}
\end{figure*}

\textbf{Communication links.} The parameters pertaining to the communication channel of the NOMA system, as discussed in \sref{Com_link}, are assigned as follows: $P_{s}$ varies from -40 to 40 dBm, 
the antenna gain for $G_k(\theta)$ and $G_{\mathcal {PS}}$ is set to 6.98 dBi, the carrier frequency $f$ is 20 GHz,  $T$ is 354.81 K, and $B$ is set to 50 MHz. The power allocation coefficients allocate 75\% and 25\% power to the FS and NS, respectively. The path loss exponent is 4, and the fading severity parameter $m$ is 2. The average power of the multi-path component $\iota$ and the LoS component $\Omega$ are set to 0.279 and 0.251, respectively. Finally, we select the communication channel type as shadowed Rician, with QPSK as the modulation type.

\textbf{Baselines.} To the best of our knowledge, the NOMA scheme has not been introduced to FL with LEO satellites as clients before. Therefore, we first investigate the performance of NomaFedHAP with a single HAP as $\mathcal {PS}$ in comparison with traditional OMA schemes. Second, we compare NomaFedHAP against state-of-the-art (SOTA) FL-LEO approaches proposed recently and reviewed in \sref{sec:releated}, including: 
\begin{itemize}
    \item {\bf Synchronous FL approaches:} FedAvg \cite{mcmahan2017communication}, FedHAP \cite{happaper}, FedISL \cite{razmi}, and DSFL \cite{wu2022dsfl}.
    \item {\bf Asynchronous FL approaches:} FedSatSchedule \cite{razmi2022scheduling}, FedSpace~\cite{so2022fedspace},  FedSat~\cite{razmi2022ground},  AsyncFLEO~\cite{mAsyFLEO}, and FedAsync \cite{xie2020asynchronous}.
\end{itemize}

\textbf{Datasets and ML models.}  To evaluate the performance of NomaFedHAP against baseline approaches, we focus on image classification. We employ commonly used model training datasets including MNIST, CIFAR-10, and CIFAR-100, which are frequently utilized in various FL-SatCom studies \cite{chen2022satellite,razmi2022scheduling, happaper,mAsyFLEO}. In addition, despite the lack of space application datasets, we utilize a real dataset of high-resolution satellite images called DeepGlobe for road extraction to provide a realistic evaluation of NomaFedHAP as well as demonstrate its applicability to real-world scenarios. 
The specifics of each dataset are as follows:
\begin{itemize}
    \item \textbf{MNIST \cite{deng2012mnist}:} is a dataset consisting of 70,000 grayscale images of handwritten numbers of size 28$\times$28 pixels. To train our satellites, we use a convolutional neural network (CNN) with three convolutional layers, three pooling layers, and one fully connected layer with 437,840 trainable parameters.

    \item \textbf{CIFAR-10 \cite{CIFAR-10}:} is a dataset consisting of 60,000 colored images of ten classes, each with 6000 images. Each image has a size of 32$\times$32 pixels (images of animals and vehicles). To train our satellites, we also use the CNN model with 798,653 training parameters.

    \item \textbf{CIFAR-100 \cite{CIFAR-10}:} is a dataset similar to CIFAR-10, but contains 100 classes with 600 images each, which makes the training task more challenging. As a result, the CNN model is constructed using six convolutional layers and two fully connected layers, with 7,759,521 training parameters.

    \item \textbf{DeepGlobe for Road Extraction \cite{DeepGlobe18}:} is a dataset consisting of 6,226   colored satellite images of size 1024$\times$1024 with a high-resolution of 50 cm/pixel. For more effective training, we apply various data augmentation techniques, such as flipping, rotating, and contrast adjusting, resulting in a dataset size of 12,452. To train our satellites, We use the U-Net model with 3,048,576 training parameters.
\end{itemize}
Our analysis considers both IID and non-IID data distributions (except for DeepGlobe where the images are already non-IID). In the IID setup, each satellite on each shell trains over the same classes of images, but these images are shuffled randomly and distributed equally across satellites. In the non-IID setup, satellites on each of two shells train on a distinct set of 30\% of the classes, while satellites on the other shell train on 40\% of the classes. The training hyperparameters are set as follows: the number of local training epochs is 100, the learning rate ranges from 0.1 to 0.0001, and the mini-batch size is 32.

\begin{figure*}[!t]
     \centering
         \subfloat[The BER vs. transmitting power.]{\centering 
         {\includegraphics[width=0.47\textwidth]{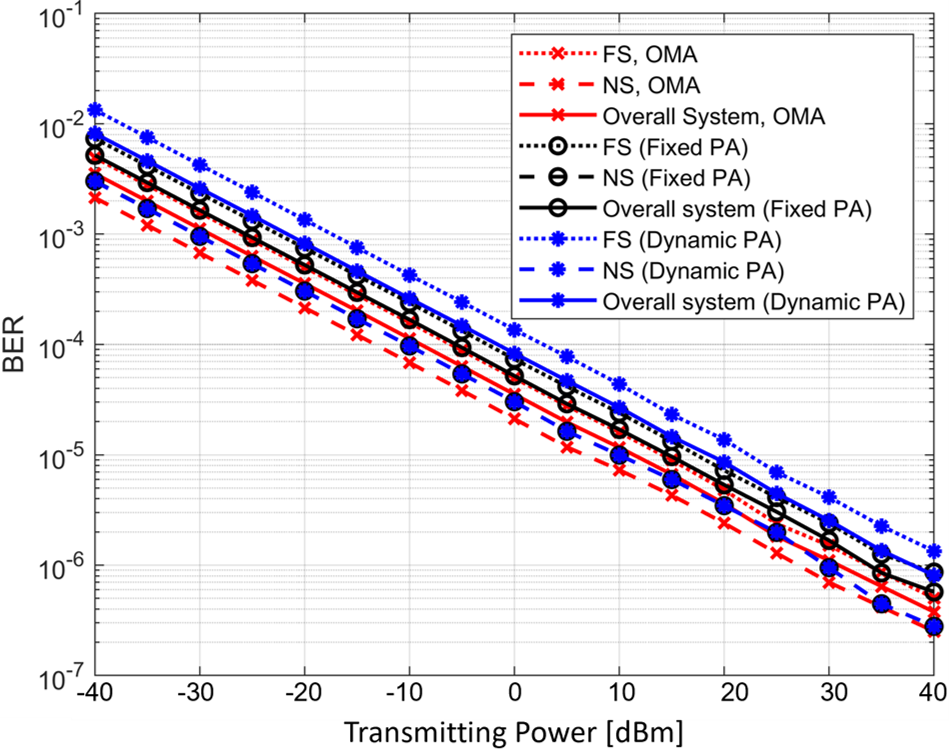}}}
     \hfill
     \subfloat[The achievable satellites capacity vs. transmitting power.]{\centering 
         \includegraphics[width=0.503\textwidth]{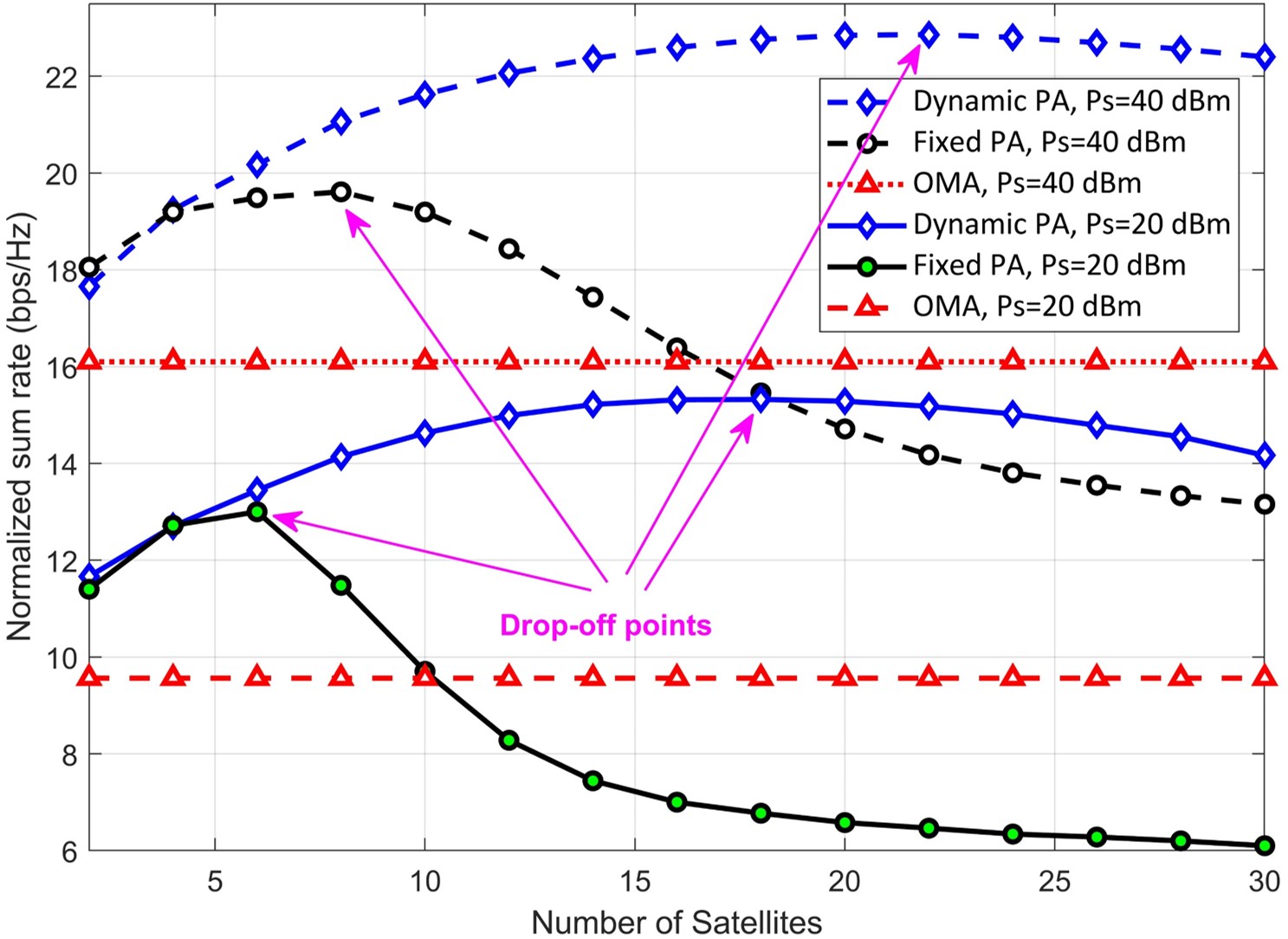}}
\hfill
        \caption{Performance of NomaFedHAP with fixed and dynamic power allocation vs. OMA scheme.}
\label{NOMA_rate}
\vspace{2mm}
     \centering
         \subfloat[The achievable data rate vs. transmitting power.]{\centering 
         {\includegraphics[width=0.5\textwidth]{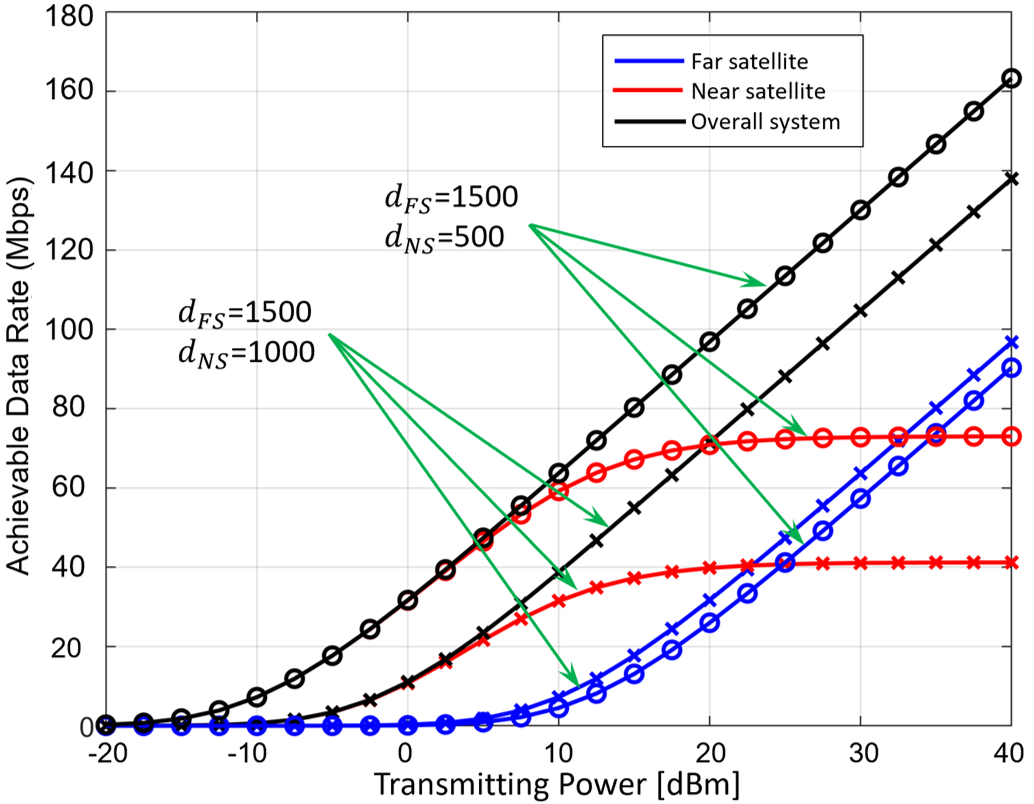}}}
     \hfill
     \subfloat[The outage probability vs. transmitting power.]{\centering 
         \includegraphics[width=0.495\textwidth]{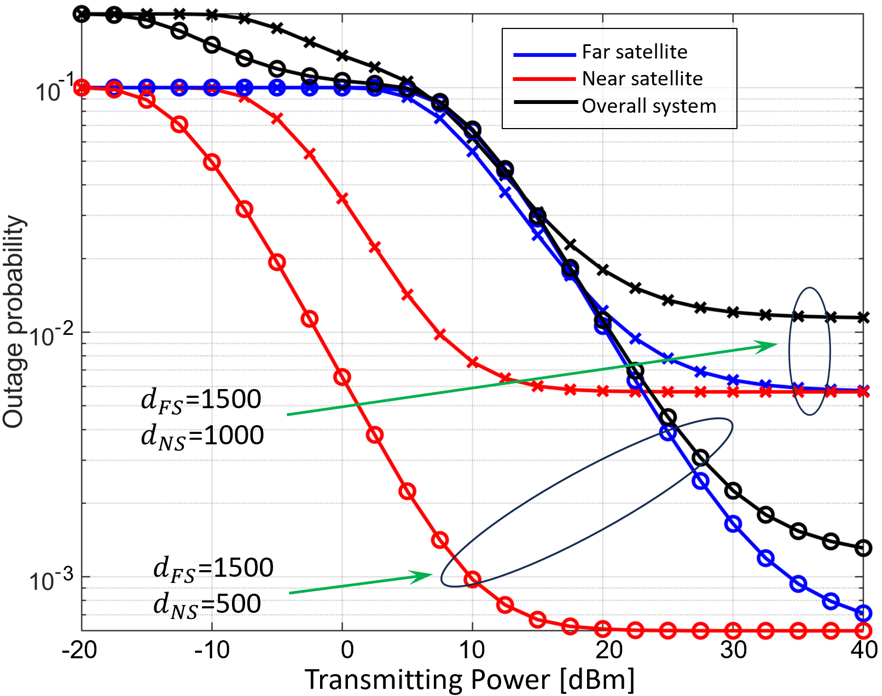}}
\hfill
        \caption{Comparison of achievable data rate and outage probability for NomaFedHAP at two altitudes (NS and FS). B=50 MHz.}
\label{NOMA_Scenario}\vspace{-3mm}
\end{figure*}

\subsection{Evaluation of NomaFedHAP's Communication Scheme}
We first compare the performance of the NomaFedHAP scheme to traditional OMA schemes with LEO satellites as clients. Figure~\ref{NOMA_rate}.a shows the BER performance against transmitting power of an NS and FS satellite at altitudes of 500km and 1500 km, respectively, for various scenarios: i) when using the NomaFedHAP with both static and dynamic power allocation (PA) based on their distance to $\mathpzc{h}$, and ii) when employing the OMA scheme. According to this figure, the OMA scheme achieves a slightly better BER compared to both PA scenarios of NomaFedHAP. This advantage is because the transmitted data from various satellites do not interfere with each other in OMA. However, despite this small improvement in the BER, the capacity of satellites that NomaFedHAP can support simultaneously at the same time and frequency 
is much higher compared to OMA. This is illustrated in Fig.~\ref{NOMA_rate}.b, which shows the achievable satellite capacity versus transmitting power of the simulated NomaFedHAP approach.

Furthermore, in Fig.~\ref{NOMA_Scenario}.a, we compare NomaFedHAP's achievable data rate versus transmitting power under various scenarios of altitudes for NS and FS, as well as for the overall system rate. From this figure, we observe that when the distance between the NS and FS is large enough, the overall system rate is higher compared to smaller distances. This is because the $\mathcal {PS}$ can easily decode the signals from distant satellites without interference. {\em Notably, in both cases, the achievable data rate ranges from 140 Mbps to 160 Mbps at a transmitting power of 40 dBm and bandwidth of 50 MHz, which is more than sufficient to transmit large models like the VGG-16 model of 528 MB. This means the uploading of models to the $\mathcal {PS}$ only takes around 30.17 to 26.4 seconds, demonstrating that employing NOMA with LEO satellites significantly reduces the required time for model uploading from minutes to just a few seconds.}

\begin{figure*}[!t]
     \centering
     \subfloat[Varying the multi-path average power component.]{\centering 
         \includegraphics[width=0.49\textwidth]{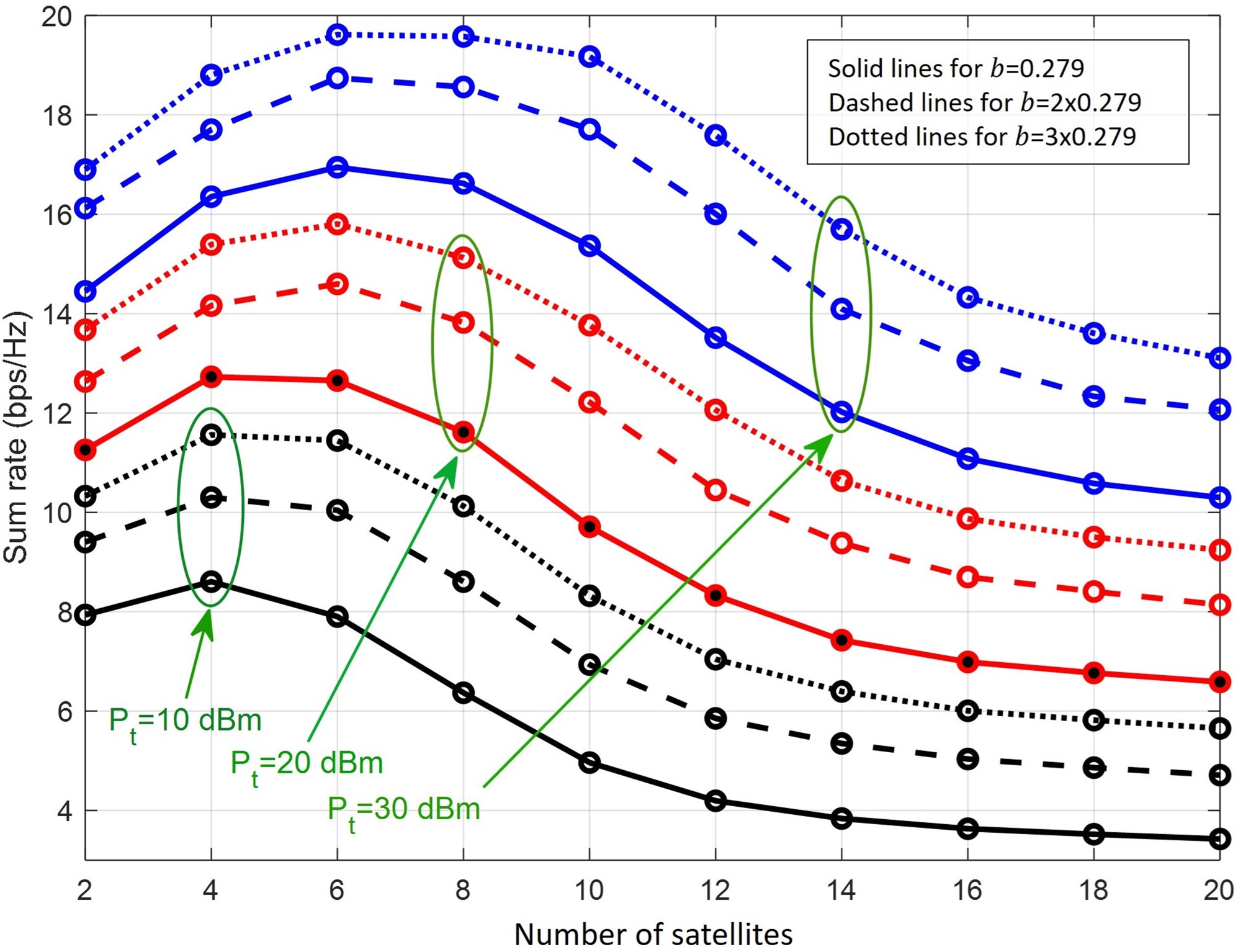}}
     \hfill
     \subfloat[Varying the LoS average power component.]{\centering 
        \includegraphics[width=0.5\textwidth]{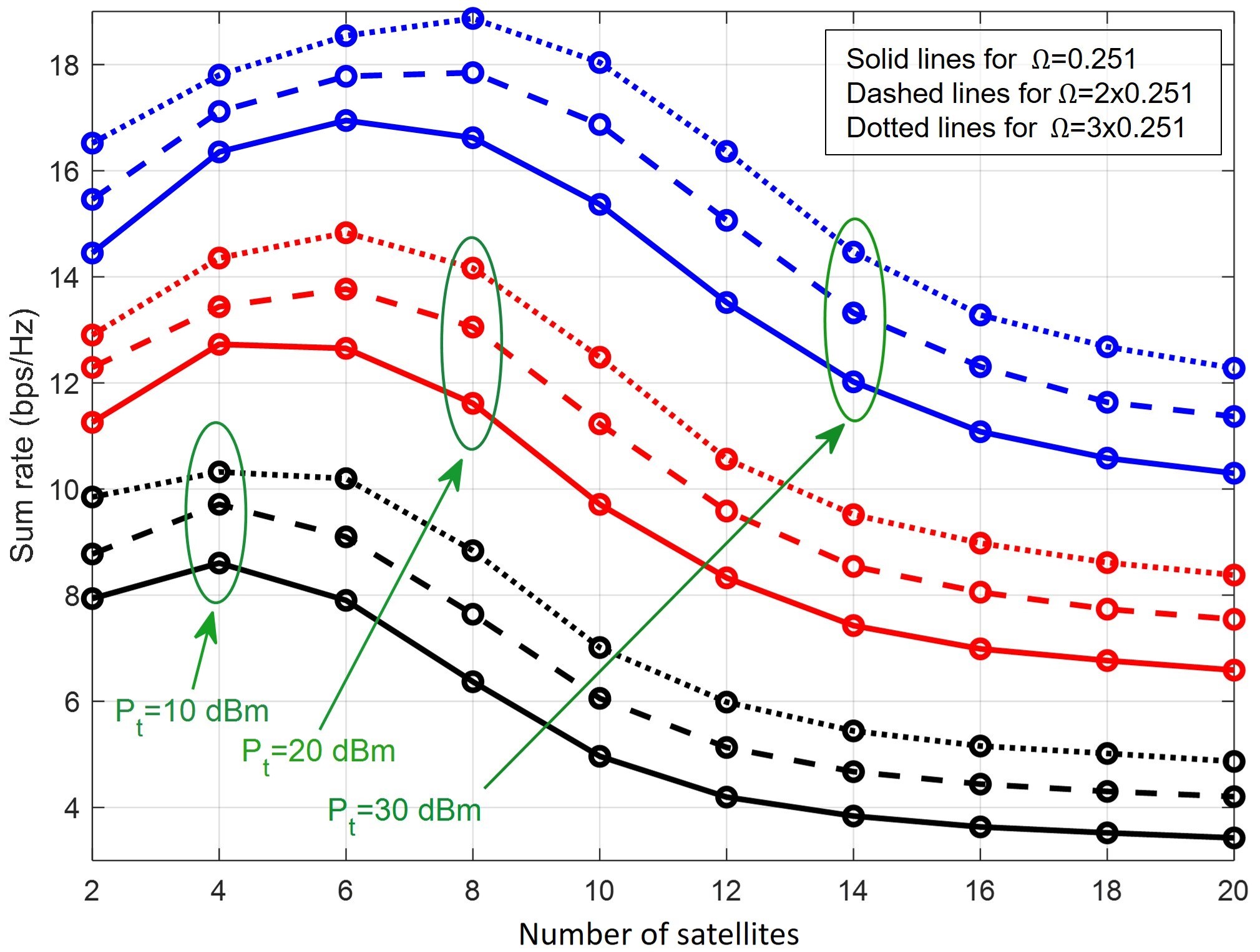}}
\hfill
        \caption{The sum data rate vs. number of supported satellites in NomaFedHAP under different settings.}
\label{satellites_rate}\vspace{-0.1cm}
\end{figure*}

In Fig.~\ref{NOMA_Scenario}.b, we show the OP experienced by $\mathpzc{h}$ versus transmitting power for the same two scenarios of the NS and FS, as well as for the overall system. According to the simulated results, the overall outage of the communication between the NS or FS to the $\mathcal {PS}$ is around  1\% chance when the transmitting power is around 20 dBm, and this decreases to 0.1\% when the transmitting power is increased to 40 dBm. These results demonstrate that NomaFedHAP achieves lower OP under various scenarios.

Finally, in Fig.~\ref{satellites_rate}, we show the total sum rate versus the maximum number of satellites that NomaFedHAP can support, considering varying multi-path and LoS average powers at different transmitting power levels. Two key observations emerge from this figure. Firstly, NomaFedHAP can support a high number of satellites even with lower multi-path or LoS average power. For instance, with a transmitting power of 30 dBm, using $\iota$=0.279 and $\Omega$=0.251,  {\em  NomaFedHAP can communicate with 14 satellites simultaneously, achieving 12 bps/Hz, and with $B$=50 MHz, each satellite will have approximately 600 Mbps. In addition, increasing $B$ will further boost the sum rate and allow for an increase in the number of supported satellites while maintaining high data rates.} Secondly, there is a drop-off point beyond which the data rate decreases. However, even after this drop-off, NomaFedHAP still supports a substantial number of satellites with high data rates. Therefore, introducing NOMA to FL-LEO significantly improves system capacity, allowing for more satellites to be launched without bandwidth limitations experienced with OMA schemes.

\subsection{Evaluation of NomaFedHAP's Convergence  Operation} 

\subsubsection{\bf Comparing NomaFedHAP with baselines} Table \ref{table_comp} presents the convergence performance of NomaFedHAP against baselines based on MNIST, CIFAR-10, and CIFAR-100 datasets in a non-IID setting with a GS in Rolla, USA serving as a $\mathcal {PS}$. Additionally, we use a GS at NP for some baselines to align with their setup environment. The table shows that NomaFedHAP achieves an accuracy of 82.73\%, 77.36\%, and 62.81\% on the MNIST, CIFAR-10, and CIFAR-100 datasets, respectively, within only 24-hour timestamp and without any impractical assumption using a GS as baselines, and not HAPs (results with HAPs are provided in Table~\ref{DataSet_com}). 

\begin{table*}[!t]
\setlength{\tabcolsep}{0.3em}
\centering
\renewcommand{\arraystretch}{1.2}
\caption{Accuracy and Convergence time of NomaFedHAP vs. baselines under non-IID setting.} 
\label{table_comp}
\resizebox{\linewidth}{!}{%
 \begin{tabular}{|p{0.35 cm}|p{2.6cm}|p{1.25cm} | p{1.85cm}| p{2cm}| p{2.2cm} | p{9cm} |}
 \hline
 \rowcolor{gray!30} &\centering FL-LEO &\multicolumn{3}{c |} {Accuracy (\%)} & \centering Convergence &  {Remark}\\
 \cline{3-5}
\rowcolor{gray!30} &\centering Approaches &\rmfamily MNIST& CIFAR-10& CIFAR-100&\centering time (h)& \\
 \hline 
\cellcolor{gray!30}\multirow{5}{*}
 & {FedAsync} \cite{xie2020asynchronous} & 70.36& 61.81& 56.37&\centering 48& GS at any location \\
    \cline{2-7}
\cellcolor{gray!30}&  FedSat  \cite{razmi2022ground} &  85.15  & 81.18 & 72.19  & \centering 24& GS located at the \textbf{NP} \\
   \cline{2-7}
 \cellcolor{gray!30}& {FedSatSched} \cite{razmi2022scheduling} & 73.61& 62.77& 54.59& \centering48& GS at any location \\
   \cline{2-7}
\cellcolor{gray!30} & FedSpace \cite{so2022fedspace}& 52.67&39.41& 36.04&\centering72 & Satellites need to upload portion of their raw data \\ 
 \cline{2-7}
\cellcolor{gray!30}\multirow{-5}{12mm}{\rotatebox[origin=p]{90}{\centering {\textbf{Async FL}}}} &AsyncFLEO \cite{mAsyFLEO}&  79.49& 69.88& 61.43&\centering 9& Assuming enough visible period for sink satellites\\
\cline{1-7}\cline{1-7}\cline{1-7}\cline{1-7}\cline{1-7}\cline{1-7}\cline{1-7}
  \cellcolor{gray!30}\multirow{6}{*}
&FedAvg \cite{mcmahan2017communication} & 79.41& 70.68& 61.66&\centering 60& GS at any location \\
\cline{2-7}
\cellcolor{gray!30}&FedISL  \cite{razmi} &  82.76& 73.62& 66.57  &\centering 8 & GS located at \textbf{NP}\\
 \cline{2-7}
\cellcolor{gray!30}&FedISL \cite{razmi} &  61.06& 52.11& 47.99&\centering72 & GS located at any location\\
  \cline{2-7}
\cellcolor{gray!30}&FedHAP \cite{happaper} &  81.62& 76.63&  59.89& \centering 48& GS located at any location \\
 \cline{2-7}
 \cellcolor{gray!30} &  
 {DSFL} \cite{wu2022dsfl}&  76.69&71.63 &62.18 & \centering 19& \small{Require higher data rates for model exchange due to Doppler shift}\\
  \cline{2-7}
 \cellcolor{gray!30} &  {FedLEO} \cite{e2023opt}&  {84.69}& 73.26& 61.31& \centering 36& GS at any location (requires scheduling sink satellite)\\
    \cline{2-7} 
\rowcolor{orange!30}  \cellcolor{gray!30}\multirow{-7.5}{12mm} {\rotatebox[origin=p]{90}{\centering{\textbf{Sync FL}}}} &  \textbf{NomaFedHAP} &  \textbf{82.73}& \textbf{77.36}& \textbf{62.81}& \centering \textbf{24}& GS located at any location \\
 \hline 
\end{tabular}}
\end{table*}

\begin{table*}[!ht]

\setlength{\tabcolsep}{0.2em}
\centering
\renewcommand{\arraystretch}{1}
\caption{Accuracy and Convergence time of NomaFedHAP under various $\mathcal{PS}_s$ scenarios.} 
\label{DataSet_com}
 \begin{tabular}{||p{1.8cm}||p{1cm} | p{1.25cm} | p{1cm} | p{1.25cm}|| p{1cm} | p{1.25cm} | p{1cm} | p{1.25cm}|| p{1cm} | p{1.25cm} | p{1cm} | p{1.25cm}||}
\hline
\hline
&\multicolumn{4}{c ||} { MNIST Dataset } &\multicolumn{4}{c ||} { CIFAR-10 Dataset}  &\multicolumn{4}{c ||} { CIFAR-100 Dataset} \\ 
\hline
 \hline
&\multicolumn{2}{c |} {\small Accuracy (\%)} &\multicolumn{2}{c||} {\centering \small Converge Time (h)}&\multicolumn{2}{c |} {\small Accuracy (\%)} &\multicolumn{2}{c||} {\centering \small Converge Time (h)}&\multicolumn{2}{c |} {\small Accuracy (\%)} &\multicolumn{2}{c||} {\centering \small Converge Time (h)}\\ 
\cline{2-13}

  $\mathcal {PS}$&  \small{ IID}& \small{Non-IID}&\small{\rmfamily IID}& \small{Non-IID}&\small{\centering IID}& \small{Non-IID}&\small{ IID}& \small{Non-IID} &\small{IID}& \small{Non-IID}&\small{IID}& \small{Non-IID}\\
 \hline 
 
GS& 97.14& 90.69&24&36  & 88.28 &80.23 &32 &42 & 76.56& 71.99&52&61\\
 \hline
Single HAP & 93.12 &90.88 &3 &9.11 & 85.19 & 82.5&4.92 &12& 78.28 &72.82 &48 &54\\
 \hline
Two HAPs & 96.23 &93.67 &1.74&3.68 & 87.67 &83.29 &3.17& 4.38 & 77.67 & 75.13&24&  30\\
 \hline
 Three HAPs & 97.62 &95.19 &1.62 & 3& 89.13&84.67&2.27& 3.81 & 80.09 & 78.62&22&26\\
 \hline 
  \hline 
\end{tabular}
\end{table*}

\begin{figure*}[!t]
     \centering
         \subfloat[MNIST dataset.]{\centering 
         {\includegraphics[width=0.335\textwidth]{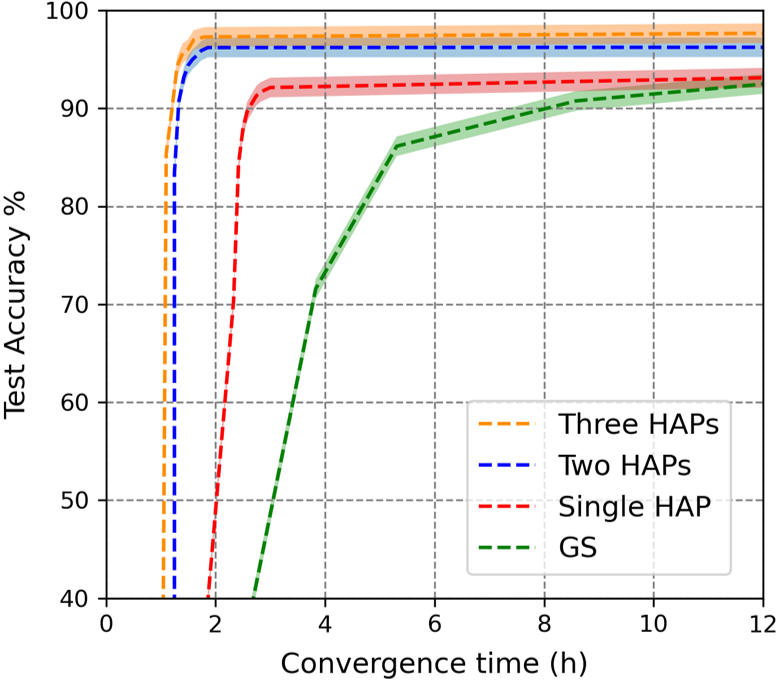}}}
     \hfill
     \subfloat[CIFAR-10 dataset.]{\centering 
         \includegraphics[width=0.33\textwidth]{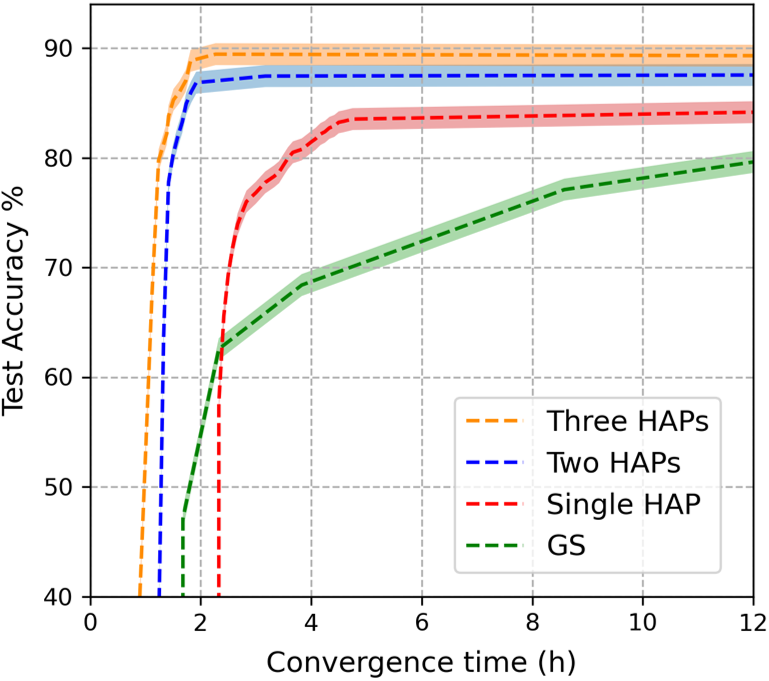}}
     \hfill
     \subfloat[CIFAR-100 dataset.]{\centering 
         \includegraphics[width=0.33\textwidth]{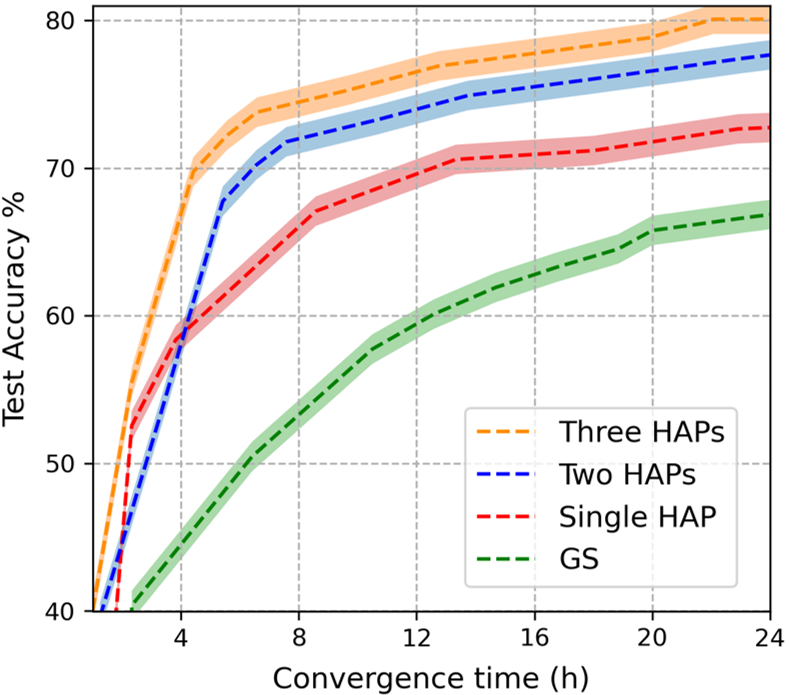}}
\hfill
        \caption{NomaFedHAP's accuracy over time for various datasets in the IID setting.}
\label{IID}
\vspace{3mm}
     \centering
         \subfloat[MNIST dataset.]{\centering 
         {\includegraphics[width=0.325\textwidth]{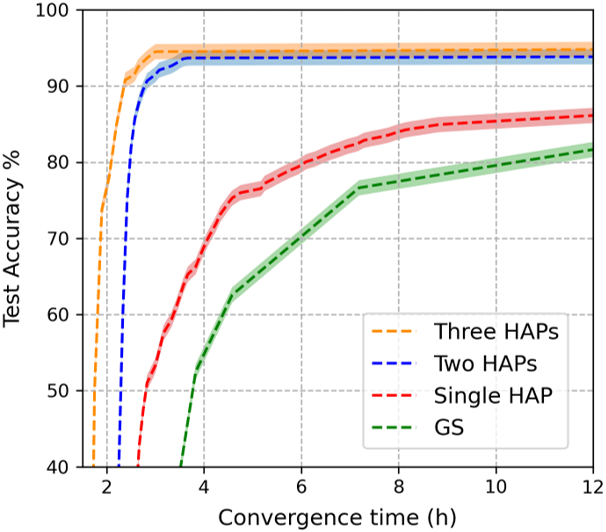}}}
     \hfill
     \subfloat[CIFAR-10 dataset.]{\centering 
         \includegraphics[width=0.32\textwidth]{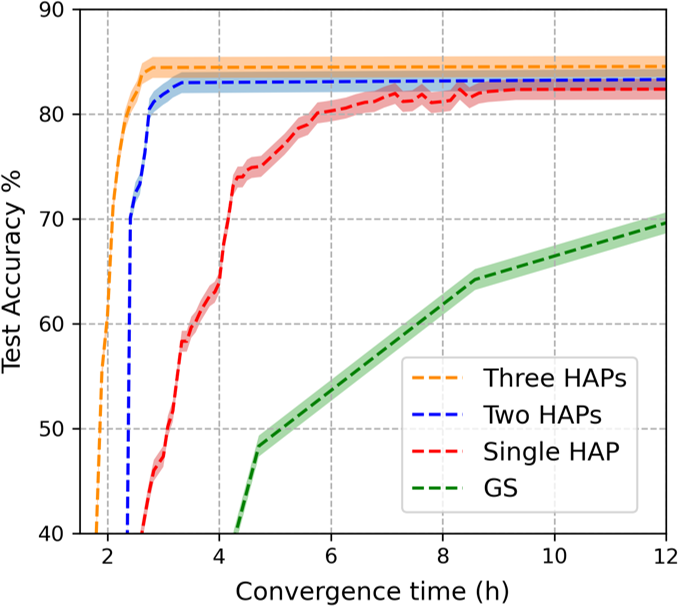}}
     \hfill
     \subfloat[CIFAR-100 dataset.]{\centering 
         \includegraphics[width=0.32\textwidth]{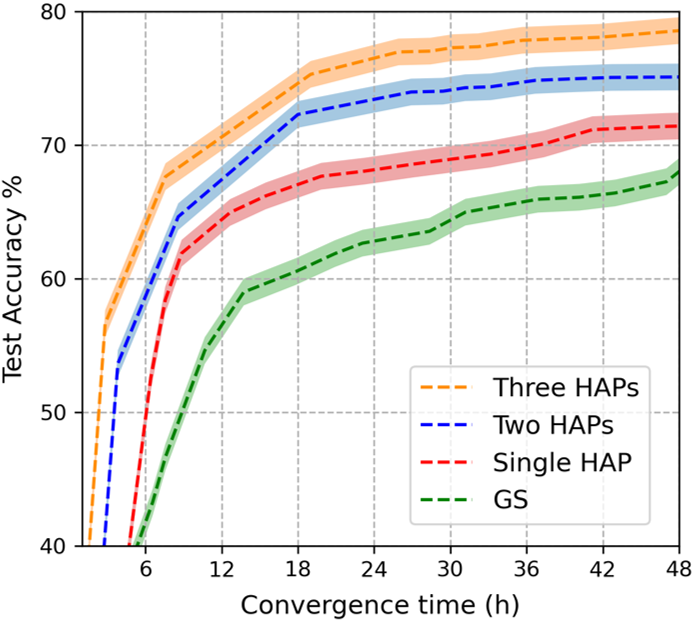}}
\hfill
        \caption{NomaFedHAP's accuracy over time for various datasets in the non-IID setting.}\vspace{-0.2cm}
\label{Non-IID}
\end{figure*}
Asynchronous approaches such as FedSat \cite{razmi2022ground}, or AsyncFLEO \cite{mAsyFLEO} can achieve better or comparable accuracy to NomaFedHAP. However, their usability in realistic scenarios is limited due to their oversimplified assumption that satellites should visit the GS in a regular manner or overlook the sink satellite's visibility period. When the authors of FedSat omitted this assumption by developing FedSatSchedule \cite{razmi2022scheduling}, the convergence speed is doubled to 48 hours, and the accuracy is reduced by 14-24\% in comparison with FedSat on different datasets. Although FedSatSchedule offers more realistic consideration, traditional FedAsync \cite{xie2020asynchronous} can achieve similar accuracy for the same convergence period. Despite FedSpace's \cite{so2022fedspace} efforts to balance idleness and staleness in synchronous approach and asynchronous FL approaches, respectively, its performance is limited to 50\% or less on the three datasets tested. 

For the synchronous FL approaches, FedISL \cite{razmi} is the fastest synchronous FL approach with a convergence time of 8 hours and an accuracy of 82.76\%, 73.62\%, and 66.57\% on MNIST, CIFAR-10, and CIFAR-100, respectively, when the GS is located at the NP. However, when the GS location is changed, its accuracy drops to 61.06\%, 52.11\%, and 47.99\% on the same datasets after 72 hours of training.  DSFL\cite{wu2022dsfl} is the second fastest, converging within 19 hours, however, it suffers from the Doppler shift due to the inter-orbit ISL communication. The comparison with the rest of the baselines is summarized in Table~\ref{table_comp}.

\begin{figure*}[!h]
\centering
    \subfloat[Samples of original satellite images.]{
         \centering
         {\includegraphics[width=\linewidth]{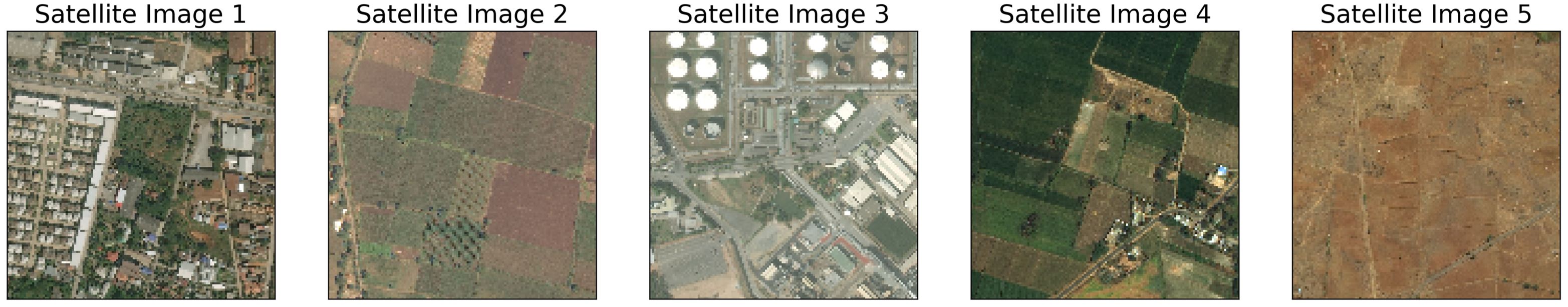}}}
  \hfill
     \subfloat[Corresponding ground truth labels of roads.]{
         \centering
         {\includegraphics[width=\linewidth]{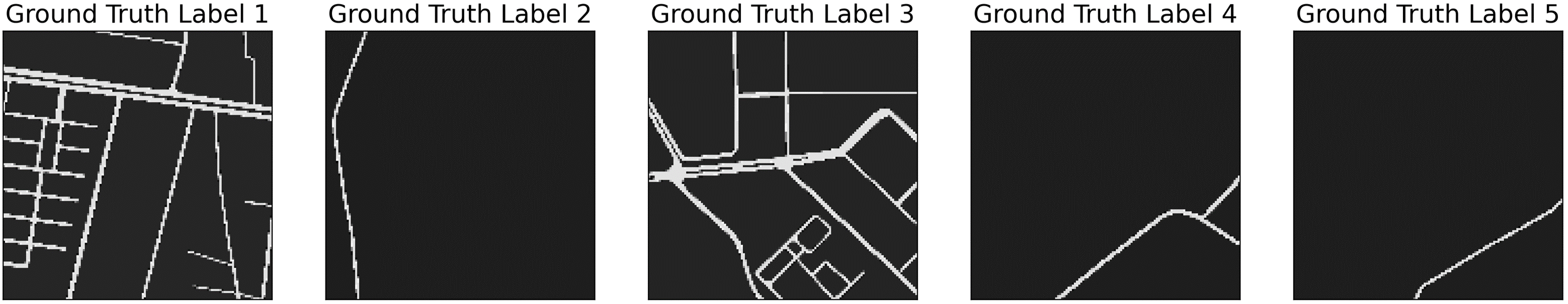}}}
  \hfill
    \subfloat[Corresponding predicated labels after 5 hours during convergence.]{
         \centering
       {\includegraphics[width=\linewidth]{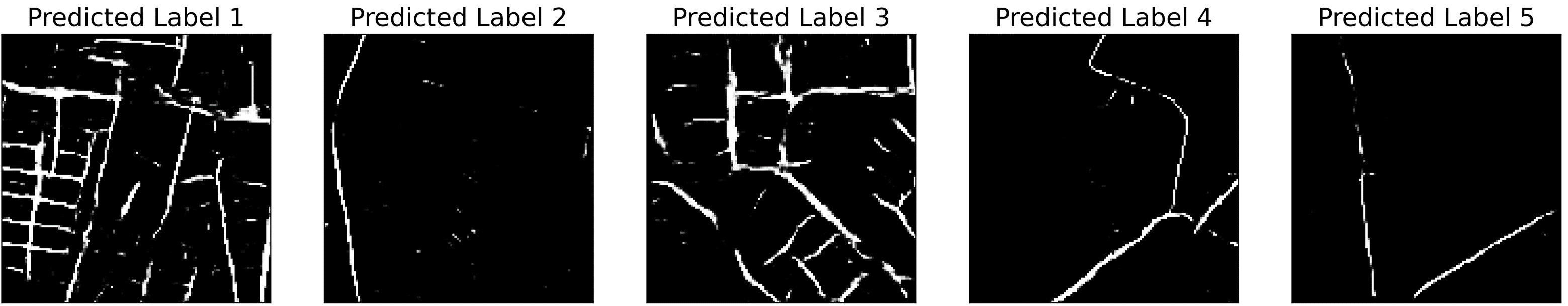}}}
  \hfill
\subfloat[Corresponding predicated labels after 10 hours during convergence.]{
         \centering
         {\includegraphics[width=\linewidth]{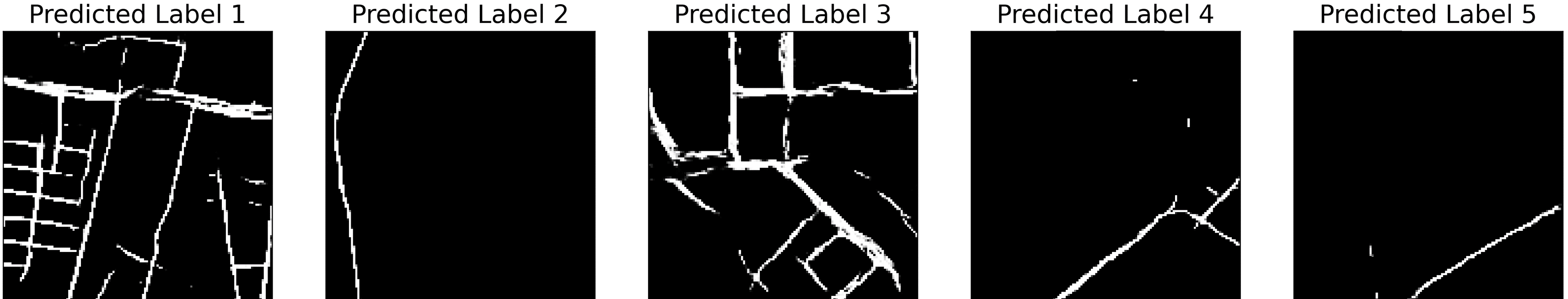}}}
        \caption{Comparison of NomaFedHAP's segmentation performance at two different timestamps (5 hours vs. 10 hours during convergence) over sample images from the DeepGlobe dataset.  The $\mathcal {PS}$ is a single HAP located at Rolla, Mo, USA.}
\label{DeepGlobe}\vspace{-0.2cm}
\end{figure*}
\subsubsection{\bf Evaluating NomaFedHAP in-depth} We extensively evaluate NomaFedHAP's performance across various scenarios, including various $\mathcal {PS}$ setups and datasets with different distribution settings. The left side of Table \ref{DataSet_com} shows the maximum achievable accuracy with respect to the convergence speed of NomaFedHAP on the MNIST dataset. The results indicate that utilizing HAP as $\mathcal {PS}$ instead of traditional $\mathcal {PS}$ significantly accelerates the convergence of the FL, reducing the convergence speed from days to just a few hours without sacrificing the target accuracy. When two/three HAPs were employed, there was a slightly different improvement in convergence speed for the IID setting compared to the single HAP scenario. However, the use of multiple HAPs had a more significant impact on the convergence speed for the non-IID settings, particularly when compared to the GS scenario. {\em These findings demonstrate that employing multiple HAPs as $\mathcal {PS}$s can enhance the convergence speed and have a substantial influence on the FL-LEO system.}
    
Table~\ref{DataSet_com} also shows the accuracy and convergence speed of NomaFedHAP using more complex datasets of CIFAR-10 and CIFAR-100. The trends observed on these datasets are similar to those seen on the MNIST dataset. According to our evaluation of NomaFedHAP's performance using these two datasets, we find that when the number of classes increased from 10 to 100, there was a reduction in accuracy of 10-15\% and an increase in convergence time of 14-45 hours. 
    
Fig.~\ref{IID} and Fig.~\ref{Non-IID} evaluate the performance of the NomaFedHAP under various evaluation conditions on a larger scale. From these figures, we can infer using even a single HAP in lieu of GS can provide higher convergence accuracy with less convergence time by an order of magnitude. This advantage also still holds under various datasets and tough distribution (non-IID), which proves its effectiveness. 

\subsubsection{\bf Evaluating NomaFedHAP using real satellite imagery} We finally evaluate the performance of NomaFedHAP on real satellite images of the DeepGlobe dataset for extracting road. To assess NomaFedHAP's effectiveness on this dataset, we use two evaluation metrics: the Intersection-over-Union (IoU) and Dice coefficient (F1 score), which are more precise than pixel accuracy for segmentation tasks. The results demonstrate that NomaFedHAP can accurately detect roads with an IoU of 61.87\% and a Dice coefficient of 65.12\% after only 5 hours. These metrics improve to 72.68\% and 73.90\%, respectively, after 10 hours. In Fig.~\ref{DeepGlobe}, we present a sample of images after 5 hours vs. 10 hours, illustrating that NomaFedHAP can quickly achieve convergence without compromising model performance across various datasets.

Comparing NomaFedHAP with one of the baseline approaches \cite{e2023opt}, that baseline initially achieves an IoU of 36.18\% and a Dice coefficient of 40.11\% after 5 hours, which subsequently improves to 69.32\% and 72.76\% after 16 hours. However, NomaFedHAP consistently achieves significantly higher accuracy—approximately 25\% more than this particular approach. Moreover, other baseline approaches exhibit significantly longer convergence times with lower accuracy under the same settings. Notably, when incorporating realistic satellite images into the training process, NomaFedHAP demonstrates the ability to expedite convergence by at least a factor of 2 to 5 when compared to those baseline FL-LEO approaches.

\section{Conclusion}\label{conclusion}
This paper introduces NomaFedHAP, a novel FL framework tailored for LEO satellite constellations, leveraging HAPs as distributed $\mathcal {PS}$s to facilitate FL model training. NomaFedHAP tackles the challenge of much prolonged FL-LEO training time due to the irregular and sporadic LEO satellite connectivity with the $\mathcal {PS}$, and transient visibility windows. To that end, NomaFedHAP introduces NOMA into FL-LEO, enabling efficient utilization of SatCom bandwidth and fast model exchange between satellites and the $\mathcal {PS}$ within just seconds. Under our new communication architecture, we also derive a closed-form expression for the outage probability of the NS and FS scenarios as well as the entire system orchestrated by HAPs. NomaFedHAP consists of i) a new communication topology utilizing HAPs as relays to mitigate Doppler shift among satellites in different orbits, ii) a novel model propagation scheme for seamless model exchange between satellites and HAPs, and iii) an optimized model aggregation approach for balancing models from different orbits and shells to achieve rapid FL convergence. Extensive simulations demonstrate NomaFedHAP's superiority in rapid and efficient FL model convergence on realistic satellite datasets, while limiting the OP of NOMA to only 0.1\%, outperforming the state-of-the-art approaches by at least 5 times in terms of convergence speed.

\begin{appendices}
\section{Proof of Theorem.1}
Let $\boldsymbol{w}_{k}^\beta$ is the local model updated by satellite $k$ during communication round $\beta$ with SGD local iterations $E\geq1$ before transmitting the updated version to the $\mathcal{PS}$. Additionally, let $\mathcal{I}_{E}$ represents the set of global synchronization steps, denoted as $\mathcal{I}_{E}= \{nE|n = 1, 2, \dots\}$. Thus, the update equation for NomaFedHAP, with partially visible satellites, can be expressed as
\begin{equation}\label{eq:new_var}
    \boldsymbol{v}_{k}^{\beta+1} = \boldsymbol{w}_{k}^{\beta}- \zeta_{{\beta}} \nabla F_{k}(\boldsymbol{w}_{k}^{\beta},\xi_{k}^{\beta})
\end{equation}
\begin{equation}
\boldsymbol{w}_{k}^{\beta+1}=
\begin{cases}
        \boldsymbol{v}_{k}^{\beta+1} \hspace{2cm}\text{if}~\beta+1 \notin \mathcal{I}_{E},\\ \sum_{k=1}^{K}\alpha_{k}  \boldsymbol{v}_{k}^{\beta+1} \hspace{0.5cm}\text{if}~\beta+1 \in \mathcal{I}_{E}.
\end{cases}
\end{equation}
In Equation~\eqref{eq:new_var}, we introduce an additional variable $\boldsymbol{v}_{k}^{\beta+1}$, which represents the immediate result of one step of SGD from $\boldsymbol{w}_{k}^{\beta}$. Here, $\boldsymbol{w}_{k}^{\beta+1}$ can be viewed as the model obtained after the aggregation round $\beta+1$, which corresponds to a global synchronization step.

Motivated by \cite{ribero2022federated}, we define two virtual sequences, namely, $\boldsymbol{\Bar {v}}^{\beta}=\sum_{k=1}^{K}\alpha_{k}\boldsymbol{\Bar v}_{k}^{\beta}$ and $\boldsymbol{\Bar w}^{\beta}=\sum_{k=1}^{K}\alpha_{k}\boldsymbol{\Bar w}_{k}^{\beta}$. We can therefore interpret $\boldsymbol{\Bar {v}}^{\beta+1}$ as the result of a single-step SGD update from $\boldsymbol{\Bar w}^{\beta}$. When $\beta+1$ is not within $\mathcal{I}_E$, both $\boldsymbol{\Bar {v}}^{\beta}$ and $\boldsymbol{\Bar {w}}^{\beta}$ are inaccessible. However, if $\beta+1$ is part of $\mathcal{I}_E$, we can only access $\boldsymbol{\Bar w}^{\beta+1}$. For convenience, we define $\mathbf{\Bar g}^{\beta}=\sum_{k=1}^{K}\alpha_{k}\nabla F_{k}(\boldsymbol{w}_{k}^{\beta})$ and $\mathbf{g}^{\beta}=\sum_{k=1}^{K}\alpha_{k}\nabla F_{k}(\boldsymbol{w}_{k}^{\beta},\xi_{k}^{\beta})$. Thus, we have $\mathbb{E}\mathbf{g}^{\beta}=\mathbf{\Bar g}^{\beta}$ and $\boldsymbol{\Bar {v}}^{\beta+1}=\boldsymbol{\Bar {w}}^{\beta}-\zeta_{{\beta}} \mathbf{\Bar g}^{\beta}$
\begin{lemma}[Results of single-step SGD] 
Assuming Assumptions 1 and 2 are satisfied, if $\zeta_{{\beta}}\leq\frac{1}{4\Lambda}$, then we have:
\begin{align}
     &\mathbb{E}\|\boldsymbol{\Bar {v}}^{\beta+1}-\boldsymbol{w}^*||^2_2\leq (1-\zeta_{{\beta}}\varrho)\mathbb{E}\|\boldsymbol{\Bar {w}}^{\beta}-\boldsymbol{w}^*||^2_2+\zeta_{{\beta}}^2\mathbb{E}\|\mathbf{g}^{\beta}-\mathbf{\Bar g}^{\beta}\|^2_2\nonumber
     \\&+6\Lambda \zeta_{{\beta}}^2\Gamma+2\mathbb{E}\sum_{k=1}^{K}\alpha_{k}\|\boldsymbol{\Bar {w}}^{\beta}-\boldsymbol{{w}}^{\beta}_{k}\|^2_2.   \nonumber
\end{align}
where $\boldsymbol{w}^*$ is the target model that achieves the desired accuracy.   
\end{lemma}
\begin{lemma}[Bounding the variance]
Assuming Assumption 3 is satisfied, then we have:
$$\mathbb{E}\|\mathbf{g}^{\beta}-\mathbf{\Bar g}^{\beta}\|^2_2\leq \sum_{k=1}^{K}\alpha_{k}^2\sigma_{k}^{2}.$$
\end{lemma}
\begin{lemma}[Bound $\boldsymbol{{w}}^{\beta}_{k}$ divergence]
Assuming Assumption 4 is satisfied and $\zeta_{\beta}\leq2\zeta_{\beta+E}$ is non-increasing $\forall~\beta\geq 1$, then we have:
$$  \mathbb{E}\Bigg[\sum_{k=1}^{K}\alpha_{k}\|\boldsymbol{{\Bar w}}^{\beta}_{k}-\boldsymbol{{w}}^{\beta}_{k}\|^2_2\Bigg]\leq4\zeta_{{\beta}}^2(E-1)^2 G^2.$$
\end{lemma}
\begin{proof}
Let $\Delta^\beta=\mathbb{E}\|\boldsymbol{\Bar {w}}^{\beta+1}_{k}-\boldsymbol{{w}}^*\|^2_2$, and $\boldsymbol{\Bar {w}}^{\beta+1}=\boldsymbol{\Bar {v}}^{\beta+1}$ in both scenarios, whether $\beta+1\in \mathcal{I}_E$ or $\beta+1 \notin \mathcal{I}_E$. Drawing upon the results of Lemmas 1-3, we have
\begin{equation}
    \Delta^{\beta+1}\leq(1-\zeta_{{\beta}} \varrho)\Delta^{\beta}+\zeta_{{\beta}}^2 Z
\end{equation}
For a decreasing step size, let $\zeta_{\beta}=\frac{\varepsilon}{\beta+ \delta}$, where $\varepsilon>\frac{1}{\varrho}$ and $\delta>0$ are chosen such that $\zeta_{1}\leq \min\{\frac{1}{\varrho}, \frac{1}{4\Lambda}\}=\frac{1}{4\Lambda}$, and $\zeta_{\beta}\leq 2\zeta_{\beta+E}$. We aim to prove that $\Delta^\beta\leq\frac{\tau}{\beta+\delta}$, where $\tau=\max\{\frac{\varepsilon^2 Z}{\varepsilon \varrho-1}, (\delta+1)\Delta^{1}\}$. \\
To achieve this, we employ the induction method. Beginning with the base case of $\beta = 1$ and considering the definition of $\tau$, we ensure its hold  for some $\beta$ as follows
\begin{align}\nonumber
    \Delta^{\beta+1}&\leq(1-\zeta_{\beta}\varrho)\Delta^{\beta}+\zeta_{\beta}^2 Z \\\nonumber
    &=\biggl(1-\frac{\varepsilon\varrho}{\beta+\delta}\biggl)\frac{\tau}{\beta+\delta}+\frac{\varepsilon^2Z}{(\beta+\delta)^2}\\\nonumber
    &=\frac{\beta+\delta-1}{(\beta+\delta)^2}\tau+\biggl[\frac{\varepsilon^2Z}{(\beta+\delta)^2}-\frac{\varepsilon\varrho-1}{(\beta+\delta)^2}\tau\biggl]\\\nonumber
    &\leq\frac{\tau}{\beta+\delta+1}\\\nonumber
\end{align}
Leveraging Assumption 2, which asserts the strong convexity of $F(\cdot)$, we obtain
$$\mathbb{E}[F(\boldsymbol{\Bar w}^{\beta})]-F^*\leq \frac{\Lambda}{2}\Delta^{\beta}\leq\frac{\Lambda}{2}\frac{\tau}{\delta+\beta}.
$$
In particular, with the choice of $\varepsilon=\frac{2}{\varrho}$ and $\delta=\max\{8\frac{\Lambda}{\varrho}-1, E\}$, and denoting $\upsilon=\frac{\Lambda}{\varrho}$, we have $\zeta_{\beta}=\frac{2}{\varrho}\frac{1}{\delta+\beta}$, therefore
$$\mathbb{E}[F(\boldsymbol{\Bar w}^{\beta})]-F^*\leq\frac{2\upsilon}{\delta+\beta}\biggl(\frac{Z}{\varrho}+2\Lambda\Delta^{1}\biggl)$$
\end{proof}\vspace{-0.9cm}
\end{appendices}
\footnotesize{\small
\bibliographystyle{IEEEtran}
\bibliography{biblio.bib}}

\begin{thebibliography}{10}
\providecommand{\url}[1]{#1}
\csname url@samestyle\endcsname
\providecommand{\newblock}{\relax}
\providecommand{\bibinfo}[2]{#2}
\providecommand{\BIBentrySTDinterwordspacing}{\spaceskip=0pt\relax}
\providecommand{\BIBentryALTinterwordstretchfactor}{4}
\providecommand{\BIBentryALTinterwordspacing}{\spaceskip=\fontdimen2\font plus
\BIBentryALTinterwordstretchfactor\fontdimen3\font minus
  \fontdimen4\font\relax}
\providecommand{\BIBforeignlanguage}[2]{{%
\expandafter\ifx\csname l@#1\endcsname\relax
\typeout{** WARNING: IEEEtran.bst: No hyphenation pattern has been}%
\typeout{** loaded for the language `#1'. Using the pattern for}%
\typeout{** the default language instead.}%
\else
\language=\csname l@#1\endcsname
\fi
#2}}
\providecommand{\BIBdecl}{\relax}
\BIBdecl

\bibitem{OliveBranch}
Q.~Fang, Z.~Zhai, S.~Yu, Q.~Wu, X.~Gong, and X.~Chen, ``Olive branch learning:
  A topology-aware federated learning framework for space-air-ground integrated
  network,'' \emph{IEEE Transactions on Wireless Communications}, 2022.

\bibitem{pachler2021updated}
N.~Pachler, I.~del Portillo, E.~F. Crawley, and B.~G. Cameron, ``An updated
  comparison of four low earth orbit satellite constellation systems to provide
  global broadband,'' in \emph{IEEE International Conference on Communications
  Workshops}.\hskip 1em plus 0.5em minus 0.4em\relax IEEE, 2021, pp. 1--7.

\bibitem{zhai2023fedleo}
Z.~Zhai, Q.~Wu, S.~Yu, R.~Li, F.~Zhang, and X.~Chen, ``{FedLEO:} an
  offloading-assisted decentralized federated learning framework for low earth
  orbit satellite networks,'' \emph{IEEE Transactions on Mobile Computing},
  2023.

\bibitem{mcmahan2017communication}
B.~McMahan, E.~Moore, D.~Ramage, S.~Hampson, and B.~A. y~Arcas,
  ``Communication-efficient learning of deep networks from decentralized
  data,'' in \emph{Artificial intelligence and statistics}.\hskip 1em plus
  0.5em minus 0.4em\relax PMLR, 2017, pp. 1273--1282.

\bibitem{xie2020asynchronous}
C.~Xie, O.~Koyejo, and I.~Gupta, ``Asynchronous federated optimization,'' in
  \emph{12$^{th}$ Wksp on Optimization for Machine Learning}, 2020.

\bibitem{chen2022satellite}
H.~Chen, M.~Xiao, and Z.~Pang, ``Satellite-based computing networks with
  federated learning,'' \emph{IEEE Wireless Communications}, vol.~29, no.~1,
  pp. 78--84, 2022.

\bibitem{razmi}
N.~Razmi, B.~Matthiesen, A.~Dekorsy, and P.~popovski, ``On-board federated
  learning for dense leo constellations,'' in \emph{IEEE International
  Conference on Communications (ICC)}, Seoul, Southkorea, May 2022.

\bibitem{happaper}
M.~Elmahallawy and T.~Luo, ``{FedHAP}: Fast federated learning for leo
  constellations using collaborative haps,'' in \emph{2022 14th International
  Conference on Wireless Communications and Signal Processing (WCSP)}, 2022,
  pp. 888--893.

\bibitem{lin2022federated}
J.~Lin, J.~Xu, Y.~Li, and Z.~Xu, ``Federated learning with dynamic aggregation
  based on connection density at satellites and ground stations,'' in
  \emph{2022 IEEE International Conference on Satellite Computing
  (Satellite)}.\hskip 1em plus 0.5em minus 0.4em\relax IEEE, 2022, pp. 31--36.

\bibitem{chen2023edge}
C.-Y. Chen, L.-H. Shen, K.-T. Feng, L.-L. Yang, and J.-M. Wu, ``Edge selection
  and clustering for federated learning in optical inter-leo satellite
  constellation,'' \emph{arXiv preprint arXiv:2303.16071}, 2023.

\bibitem{wu2022dsfl}
C.~Wu, Y.~Zhu, and F.~Wang, ``{DSFL}: Decentralized satellite federated
  learning for energy-aware leo constellation computing,'' in \emph{2022 IEEE
  International Conference on Satellite Computing (Satellite)}.\hskip 1em plus
  0.5em minus 0.4em\relax IEEE, 2022, pp. 25--30.

\bibitem{e2023opt}
\BIBentryALTinterwordspacing
M.~Elmahallawy and T.~Luo, ``Optimizing federated learning in leo satellite
  constellations via intra-plane model propagation and sink satellite
  scheduling,'' in \emph{ICC 2023-IEEE International Conference on
  Communications}, 2023. [Online]. Available:
  \url{{https://arxiv.org/abs/2302.13447}}
\BIBentrySTDinterwordspacing

\bibitem{razmi2022ground}
N.~Razmi, B.~Matthiesen, A.~Dekorsy, and P.~Popovski, ``Ground-assisted
  federated learning in leo satellite constellations,'' \emph{IEEE Wireless
  Communications Letters}, 2022.

\bibitem{razmi2022scheduling}
------, ``Scheduling for ground-assisted federated learning in leo satellite
  constellations,'' \emph{arXiv preprint arXiv:2206.01952}, 2022.

\bibitem{so2022fedspace}
J.~So, K.~Hsieh, B.~Arzani, S.~Noghabi, S.~Avestimehr, and R.~Chandra,
  ``Fedspace: An efficient federated learning framework at satellites and
  ground stations,'' \emph{arXiv preprint arXiv:2202.01267}, 2022.

\bibitem{wang2022fl}
P.~Wang, H.~Li, and B.~Chen, ``Fl-task-aware routing and resource reservation
  over satellite networks,'' in \emph{GLOBECOM 2022-2022 IEEE Global
  Communications Conference}.\hskip 1em plus 0.5em minus 0.4em\relax IEEE,
  2022, pp. 2382--2387.

\bibitem{mAsyFLEO}
M.~Elmahallawy and T.~Luo, ``{AsyncFLEO}: {A}synchronous federated learning for
  {LEO} satellite constellations with high-altitude platforms,'' in \emph{2022
  IEEE International Conference on Big Data (BigData)}, 2022, pp. 5478--5487.

\bibitem{wu2023fedgsm}
L.~Wu and J.~Zhang, ``Fedgsm: Efficient federated learning for leo
  constellations with gradient staleness mitigation,'' \emph{arXiv preprint
  arXiv:2304.08537}, 2023.

\bibitem{elone}
M.~Elmahallawy and T.~Luo, ``One-shot federated learning for {LEO}
  constellations that reduces convergence time from days to 90 minutes,'' in
  \emph{2023 24th IEEE International Conference on Mobile Data Management
  (MDM)}, 2023, pp. 45--54.

\bibitem{elmahallawy2023secure}
\BIBentryALTinterwordspacing
M.~Elmahallawy, T.~Luo, and M.~I. Ibrahem, ``Secure and efficient federated
  learning in {LEO} constellations using decentralized key generation and
  on-orbit model aggregation,'' in \emph{GLOBECOM 2023-2023 IEEE Global
  Communications Conference}, 2023. [Online]. Available:
  \url{{https://arxiv.org/abs/2302.13447}}
\BIBentrySTDinterwordspacing

\bibitem{el2019performance}
M.~El-Mahallawy, A.~S. TagEldien, and S.~S. Elagooz, ``Performance enhancement
  of underwater acoustic ofdm communication systems,'' \emph{Wireless Personal
  Communications}, vol. 108, pp. 2047--2057, 2019.

\bibitem{gamal2022performance}
C.~Gamal, K.~An, X.~Li, V.~G. Menon, G.~Ragesh, M.~M. Fouda, and B.~M.
  ElHalawany, ``Performance of hybrid satellite-uav noma systems,'' in
  \emph{ICC 2022-IEEE International Conference on Communications}.\hskip 1em
  plus 0.5em minus 0.4em\relax IEEE, 2022, pp. 189--194.

\bibitem{yan2019ergodic}
X.~Yan, H.~Xiao, K.~An, G.~Zheng, and S.~Chatzinotas, ``Ergodic capacity of
  noma-based uplink satellite networks with randomly deployed users,''
  \emph{IEEE Systems Journal}, vol.~14, no.~3, pp. 3343--3350, 2020.

\bibitem{alsharoa2020improvement}
A.~Alsharoa and M.-S. Alouini, ``Improvement of the global connectivity using
  integrated satellite-airborne-terrestrial networks with resource
  optimization,'' \emph{IEEE Transactions on Wireless Communications}, vol.~19,
  no.~8, pp. 5088--5100, 2020.

\bibitem{jia2020joint}
Z.~Jia, M.~Sheng, J.~Li, D.~Zhou, and Z.~Han, ``Joint hap access and leo
  satellite backhaul in 6g: Matching game-based approaches,'' \emph{IEEE
  Journal on Selected Areas in Communications}, vol.~39, no.~4, pp. 1147--1159,
  2020.

\bibitem{xing2021high}
Y.~Xing, F.~Hsieh, A.~Ghosh, and T.~S. Rappaport, ``High altitude platform
  stations (haps): Architecture and system performance,'' in \emph{2021 IEEE
  93rd Vehicular Technology Conference (VTC2021-Spring)}.\hskip 1em plus 0.5em
  minus 0.4em\relax IEEE, 2021, pp. 1--6.

\bibitem{hsieh2020uav}
F.~Hsieh, F.~Jardel, E.~Visotsky, F.~Vook, A.~Ghosh, and B.~Picha, ``Uav-based
  multi-cell haps communication: System design and performance evaluation,'' in
  \emph{GLOBECOM 2020-2020 IEEE Global Communications Conference}.\hskip 1em
  plus 0.5em minus 0.4em\relax IEEE, 2020, pp. 1--6.

\bibitem{hap}
GSMA, ``High altitude platform systems: Towers in the skies,'' \emph{Technical
  report, GSMA}, 2021.

\bibitem{kurt2021vision}
G.~K. Kurt, M.~G. Khoshkholgh, S.~Alfattani, A.~Ibrahim, T.~S. Darwish, M.~S.
  Alam, H.~Yanikomeroglu, and A.~Yongacoglu, ``A vision and framework for the
  high altitude platform station (haps) networks of the future,'' \emph{IEEE
  Communications Surveys \& Tutorials}, vol.~23, no.~2, pp. 729--779, 2021.

\bibitem{marriott2020trajectory}
J.~Marriott, B.~Tezel, Z.~Liu, and N.~E. Stier-Moses, ``Trajectory optimization
  of solar-powered high-altitude long endurance aircraft,'' in \emph{2020 6th
  International Conference on Control, Automation and Robotics (ICCAR)}.\hskip
  1em plus 0.5em minus 0.4em\relax IEEE, 2020, pp. 473--481.

\bibitem{aldababsa2018tutorial}
M.~Aldababsa, M.~Toka, S.~G{\"o}k{\c{c}}eli, G.~K. Kurt, and O.~Kucur, ``A
  tutorial on nonorthogonal multiple access for 5g and beyond,'' \emph{Wireless
  communications and mobile computing}, vol. 2018, 2018.

\bibitem{bankey2021physical}
V.~Bankey, P.~K. Upadhyay, and D.~B.~d. Costa, ``Physical layer security in
  hybrid satellite-terrestrial relay networks,'' in \emph{Physical Layer
  Security}.\hskip 1em plus 0.5em minus 0.4em\relax Springer, 2021, pp. 1--28.

\bibitem{VANASSCHE}
\BIBentryALTinterwordspacing
W.~{Van Assche}, ``Ordinary special functions,'' in \emph{Encyclopedia of
  Mathematical Physics}, J.-P. Françoise, G.~L. Naber, and T.~S. Tsun,
  Eds.\hskip 1em plus 0.5em minus 0.4em\relax Oxford: Academic Press, 2006, pp.
  637--645. [Online]. Available:
  \url{https://www.sciencedirect.com/science/article/pii/B0125126662003953}
\BIBentrySTDinterwordspacing

\bibitem{zwillinger2007table}
D.~Zwillinger and A.~Jeffrey, \emph{Table of integrals, series, and
  products}.\hskip 1em plus 0.5em minus 0.4em\relax Elsevier, 2007.

\bibitem{Li2020On}
\BIBentryALTinterwordspacing
X.~Li, K.~Huang, W.~Yang, S.~Wang, and Z.~Zhang, ``On the convergence of fedavg
  on non-iid data,'' in \emph{International Conference on Learning
  Representations}, 2020. [Online]. Available:
  \url{https://openreview.net/forum?id=HJxNAnVtDS}
\BIBentrySTDinterwordspacing

\bibitem{ribero2022federated}
M.~Ribero, H.~Vikalo, and G.~De~Veciana, ``Federated learning under
  intermittent client availability and time-varying communication
  constraints,'' \emph{IEEE Journal of Selected Topics in Signal Processing},
  vol.~17, no.~1, pp. 98--111, 2022.

\bibitem{walker1984satellite}
J.~G. Walker, ``Satellite constellations,'' \emph{Journal of the British
  Interplanetary Society}, vol.~37, p. 559, 1984.

\bibitem{deng2012mnist}
L.~Deng, ``The mnist database of handwritten digit images for machine learning
  research,'' \emph{IEEE Signal Processing Magazine}, vol.~29, no.~6, pp.
  141--142, 2012.

\bibitem{CIFAR-10}
A.~Krizhevsky, V.~Nair, and G.~Hinton, ``Cifar-10 (canadian institute for
  advanced research),'' \emph{URL http://www. cs. toronto. edu/kriz/cifar.
  html}, vol.~5, no.~4, p.~1, 2010.

\bibitem{DeepGlobe18}
I.~Demir, K.~Koperski, D.~Lindenbaum, G.~Pang, J.~Huang, S.~Basu, F.~Hughes,
  D.~Tuia, and R.~Raskar, ``Deepglobe 2018: A challenge to parse the earth
  through satellite images,'' in \emph{Computer Vision and Pattern Recognition
  (CVPR) Workshops}, June 2018.

\end{thebibliography}

\end{document}